\documentclass{article} 
\usepackage{iclr,times}

\usepackage{amsmath,amsfonts,bm}

\def\1{\bm{1}}

\DeclareMathAlphabet{\mathsfit}{\encodingdefault}{\sfdefault}{m}{sl}
\SetMathAlphabet{\mathsfit}{bold}{\encodingdefault}{\sfdefault}{bx}{n}

\usepackage{amsfonts,amsmath,amssymb,amsthm}
\usepackage{mathtools}
\usepackage[mathcal]{eucal}
\usepackage{nicefrac}
\usepackage{dirtytalk}
\usepackage{comment}
\usepackage{hyperref} 
\hypersetup{colorlinks=true, 
linkcolor=black,
citecolor=black,
filecolor=black,
urlcolor=black}

\usepackage{soul}

\newtheorem{definition}{Definition}
\newtheorem{proposition}{Proposition}
\newtheorem{theorem}{Theorem}
\newtheorem{lemma}{Lemma}
\newtheorem{remark}{Remark}

\newcommand{\bv}{{\big\Vert}}

\newcommand{\cst}{{{\rm cst}}}
\newcommand{\sigmam}{(\sigma,m)}
\newcommand{\tsigmatau}{A}
\newcommand{\diag}{{\rm diag}}
\newcommand{\tr}{{{\rm tr}\,}}
\newcommand{\cov}{{{\rm cov}}}
\newcommand{\mcmc}{{\textrm{MCMC}}}

\newcommand{\logsumexp}{{\mathsf{logsumexp}}} 

\newcommand{\hess}{H}

\newcommand{\mypreceq}{\preccurlyeq}
\newcommand{\mysucceq}{\succcurlyeq}
\newcommand{\myleq}{\leqslant}
\newcommand{\mygeq}{\geqslant}
\newcommand{\myhat}{\hat}

\newcommand{\axis}{\theta}

\def\<{\begin{equation}}
\def\>{\end{equation}}

\numberwithin{equation}{section}

\usepackage{graphicx,subfigure}
\usepackage[ruled, lined, longend, linesnumbered]{algorithm2e}
\usepackage{algorithmic}
\usepackage[utf8]{inputenc} 
\usepackage[T1]{fontenc}    
\usepackage{hyperref}       
\usepackage{url}            
\usepackage{booktabs}       
\usepackage{amsfonts}       
\usepackage{nicefrac}       
\usepackage{microtype}      
\usepackage{xcolor}         

\usepackage{hyperref}
\usepackage{url}

\title{Chain of Log-Concave Markov Chains}

\newcommand{\upstairs}[1]{\textsuperscript{#1}}
\newcommand{\affilone}{1}
\newcommand{\affiltwo}{2}

\author{Saeed Saremi\upstairs{\affilone}, Ji Won Park\upstairs{\affilone}, Francis Bach\upstairs{\affiltwo} \quad  \\
    \upstairs{\affilone}Frontier Research, Prescient Design, Genentech, South San Francisco, CA \\
    \upstairs{\affiltwo}Inria, Ecole Normale Sup\'erieure, Universit\'e PSL, Paris, France \\
}

\iclrfinalcopy 
\begin{document}

\maketitle

\begin{abstract}
We introduce a theoretical framework for sampling from unnormalized densities based on a smoothing scheme that uses an isotropic Gaussian kernel with a single fixed noise scale. We prove one can decompose sampling from a density (minimal assumptions made on the density) into a sequence of sampling from log-concave conditional densities via accumulation of noisy measurements with equal noise levels. Our construction is unique in that it keeps track of a history of samples, making it non-Markovian as a whole, but it is lightweight algorithmically as the history only shows up in the form of a running empirical mean of samples. Our sampling algorithm generalizes walk-jump sampling~\citep{saremi2019neural}. The ``walk'' phase becomes a (non-Markovian) chain of (log-concave) Markov chains. The ``jump'' from the accumulated measurements is obtained by empirical Bayes. We study our sampling algorithm quantitatively using the 2-Wasserstein metric and compare it with various Langevin MCMC algorithms. We also report a remarkable capacity of our algorithm to ``tunnel'' between modes of a distribution. 
    
\end{abstract}

\section{Introduction}
Markov chain Monte Carlo (MCMC) is an important class of general-purpose algorithms for sampling from an unnormalized probability density of the form $p(x) =e^{-f(x)}/Z$ in $\mathbb{R}^d$. This is a fundamental problem and appears in a variety of fields, e.g., statistical physics going back to 1953~\citep{metropolis1953equation}, Bayesian inference~\citep{neal1995Bayesian}, and molecular dynamics simulations~\citep{leimkuhler2015molecular}. The biggest challenge facing MCMC is that the distributions of interest lie in very high dimensions and are far from being log-concave, therefore the probability mass is concentrated in small pockets separated by vast empty spaces. These large regions with small probability mass make navigating the space using Markov chains very slow. The second important challenge facing MCMC is that the log-concave pockets themselves are typically ill-conditioned\textemdash highly elongated, spanning different directions for different pockets\textemdash which only adds to the complexity of sampling.  


The framework we develop in this paper aims at addressing these problems. The general philosophy here is that of {\bf smoothing}, by which we expand the space from $\mathbb{R}^d$ to $\mathbb{R}^{md}$ for some integer $m$ and ``fill up'' the empty space iteratively with probability mass in an approximately isotropic manner, the degree of which we can control using a single smoothing (noise) hyperparameter $\sigma$. The map from (noisy samples in) $\mathbb{R}^{md}$ back to (clean samples in) $\mathbb{R}^d$ is based on the {\bf empirical Bayes} formalism~\citep{robbins1956empirical, miyasawa1961empirical, saremi2019neural, saremi2022multimeasurement}. In essence, a single ``jump'' using the empirical Bayes estimator removes the masses that were created during sampling. We prove a general result that, for any large $m$, the problem of sampling in $\mathbb{R}^{d}$ can be reduced to sampling from a sequence of log-concave densities: {\bf once log-concave, always log-concave.} The trade-off here is  the linear time cost of accumulating noisy measurements over $m$ iterations.

More formally, instead of sampling from $p(x)$, we sample from the density $p(y_{1:m})$ that is associated with $Y_{1:m}\coloneqq (Y_1,\dots,Y_m)$, where $Y_t = X+N_t$, $N_t \sim \mathcal{N}(0,\sigma^2 I)$, all independent for $t$ in $[m]$. As we show in the paper, there is a duality between sampling from $p(x)$ and sampling from $p(y_{1:m})$ in the regime where $m^{-1/2}\sigma$ is small, irrespective of how large $\sigma$ is. This is related to the notion of {\bf universality} class underlying the smoothed densities. 
Crucial to our formalism is keeping track of the history of all the noisy samples generated along the way using the factorization
\< \label{eq:oneatatime} p(y_{1:m}) = p(y_1) \prod_{t=2}^m p(y_t|y_{1:t-1}).\>
 An important element of this sampling scheme is therefore {\bf non-Markovian}.  However, related to our universality results, this history only needs to be tracked in the form of an empirical mean, so {the memory footprint is minimal} from an algorithmic perspective. See \autoref{fig:schematic} for a schematic. 

\begin{figure}[t!]
\begin{center}
\includegraphics[trim={0 0 0 0},clip,width=0.9\textwidth]{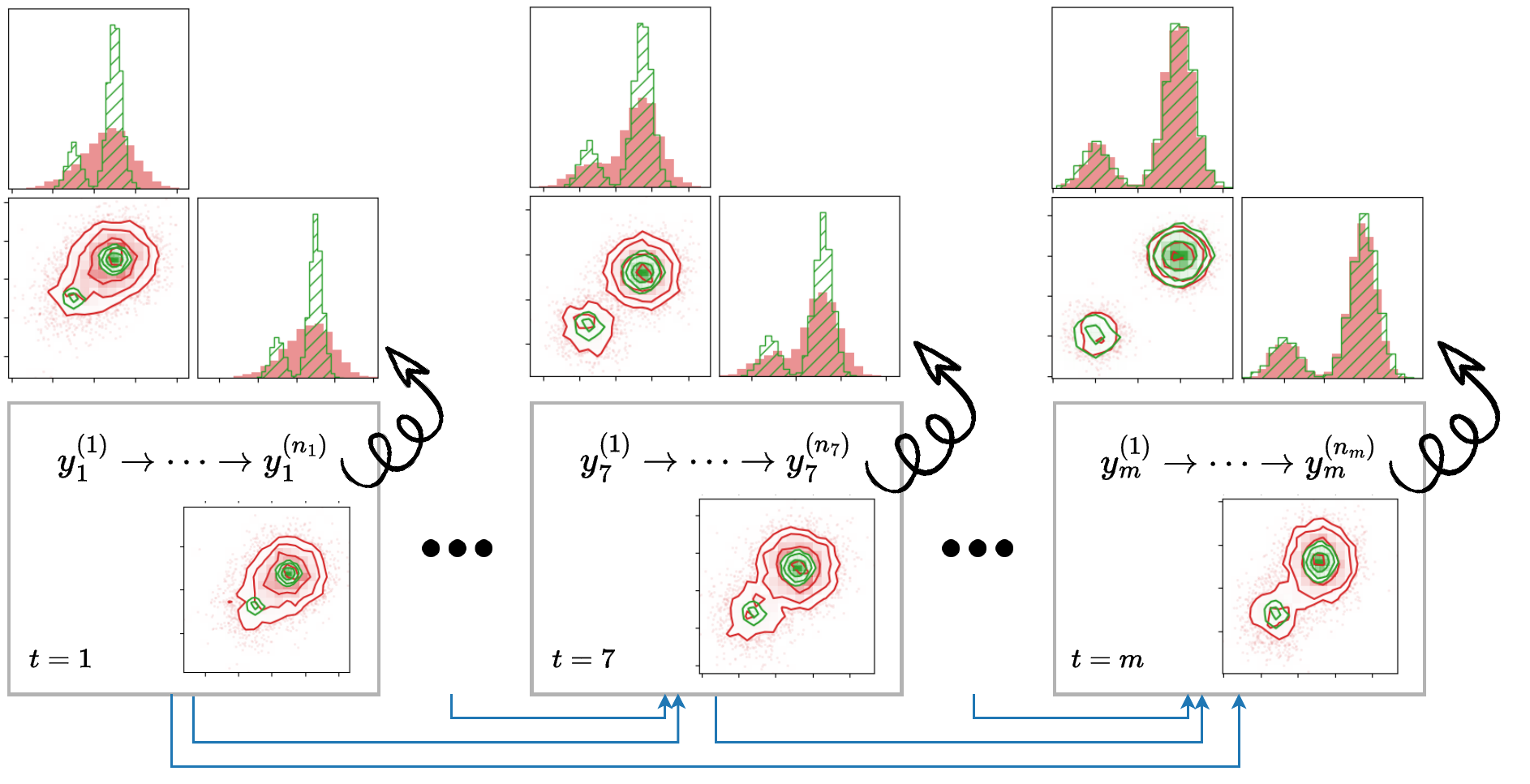}
\end{center}

\vspace*{-.2cm}

\caption{\label{fig:schematic} 
Chain of log-concave Markov chains. Here, $(y_t^{(i)})_{i \in [n_t]}$ are samples from a Markov chain, which is used to generate independent draws from $p(y_t|y_{1:t-1})$ for $t \in [m]$. The blue arrows indicate the non-Markov aspect of our sampling scheme: the accumulation of noisy measurements. The wiggly arrows indicate the denoising ``jumps''. In this example, $p(y_t|y_{1:t-1})$ is log-concave for all $t$, but the jumps asymptotically sample the target density (a mixture of two Gaussians) as $t$ increases.
}
\end{figure}

A more technical summary of our contributions and the outline of the paper are as follows:
\begin{itemize}
    \item In \autoref{sec:universal}, we prove {\bf universality} results underlying the smoothed densities $p(y_{1:m})$. 
    \item We study anisotropic Gaussians in \autoref{sec:geometry}, proving a  \emph{negative result} regarding the condition number of $p(y_{1:m})$ in comparison to $p(y_1)$ in the same universality class. This analysis becomes a segue to our factorization  \eqref{eq:oneatatime}, where in remarkable contrast we show that the condition number {\bf monotonically improves} upon accumulation of measurements.
    \item \autoref{sec:log-concave} is at the heart of the paper, where we prove several results  culminating in \autoref{theorem:compact}, which shows that a broad class of sampling problems can be transformed into a sequence of sampling from strongly {\bf log-concave} distributions using our measurement accumulation scheme. (This is a feasibility result; in particular, we do not prove here that the log-concave sampling strategy is optimal.) We examine the theorem by algebraically studying an example of a mixture of Gaussians in detail. In \autoref{sec:log-concave}, we also outline our general sampling algorithm.
    \item We validate our algorithm on carefully designed test densities in \autoref{sec:experiments}. In particular, our algorithm results in lower 2-Wasserstein metric compared to sampling from $p(x)$ using Langevin MCMC (without any smoothing). We also qualitatively report the capacity of our log-concave sampling scheme to {\bf tunnel} to a mode of a distribution with a small probability mass in a small number of steps when it is initialized at a mode with much higher mass.

\end{itemize}
\subsection{Related Work}
Our solution, sketched above, has its roots in  {\bf walk-jump} sampling~\citep{saremi2019neural} and its recent generalization~\citep{saremi2022multimeasurement}. Both papers were framed within the context of generative modeling, i.e., sampling from an \emph{unknown} distribution from which one has access to independent samples.  In contrast, this work lays the theoretical foundation for the fundamental problem of sampling from an unnormalized density when there are no samples available. In addition, regarding the recent development, we show analytically that the intuition expressed by~\citet{saremi2022multimeasurement} regarding the distribution $p(y_{1:m})$ being well-conditioned is not correct. This nontrivial negative result motivates our analysis of the non-Markovian scheme for sampling $p(y_{1:m})$.

Our methodology is agnostic to the algorithm used for sampling from $p(y_t|y_{1:t-1})$ in \eqref{eq:oneatatime}. However, we have been particularly motivated by the research on {\bf Langevin MCMC} which is a class of gradient-based sampling algorithms obtained by discretizing the Langevin diffusion~\citep{parisi1981correlation}. There is a growing body of work on the analysis of Langevin MCMC algorithms of various complexity (overdampled, Metropolis-adjusted, underdamped, higher-order) for sampling from log-concave distributions~\citep{dalalyan2017theoretical, durmus2017nonasymptotic, cheng2018underdamped, dwivedi2018log, shen2019randomized, cao2020complexity, mou2021high, li2022sqrt}. In our experiments, we compare several recent Langevin MCMC methods, which may be of independent interest.

There is a significant body of work on {\bf sequential} methods for sampling, rooted in annealing methods in optimization~\citep{kirkpatrick1983optimization}, which became popular in the MCMC literature due to Neal's seminal paper on annealed importance sampling~\citep{neal2001annealed}. Diffusion models~\citep{sohl2015deep, ho2020denoising} are a related class of sequential methods for generative modeling. Our sequential scheme is distinguished from earlier methods from separate angles: {\bf (I)} Although we sample the conditional densities with Markov chains, we condition on all the previous samples that were generated. As a whole, our scheme is strictly \emph{non-Markovian}. {\bf (II)} In our sequential scheme, we are able to guarantee that we sample from (progressively more)  \emph{log-concave} densities. To our knowledge, no other sampling frameworks can make such guarantees. {\bf (III)} Compared to diffusion models, the noise level in our framework is held \emph{fixed}. This is an important feature of our sampling algorithm and it underlies many of its theoretical properties. {\bf (IV) }  All prior sequential schemes rely on a noising/annealing \emph{schedule} which is hard to tune, and their performance is sensitive to the choice of the schedule~\citep{karras2022elucidating, syed2022non}.  In contrast, our sequential scheme is free of scheduling and relies on only two parameters: the noise level $\sigma$  and the number of measurements $m$.

\paragraph{Notation.} We use $p$ to denote probability density functions and adopt the convention where we drop the random variable subscript to $p$ when the arguments are present, e.g., $p(x)\coloneqq p_X(x)$, $p(y_2|y_1) \coloneqq p_{Y_2|Y_1=y_1}(y_2)$. We reserve $f$ to be the energy function associated with  $p(x) \propto e^{-f(x)}$. We use $\lambda$ to denote the spectrum of a matrix, e.g., $\lambda_{\max}(C)$ is the largest eigenvalue of $C$. We use the shorthand notations $[m]=\{1,\dots,m\}$, $y_{1:m} = (y_1,\dots,y_m)$, and $\overline{y}_{1:m}=\frac{1}{m} \sum_{t=1}^m y_t$. 

\section{Universal $\sigmam$-densities} \label{sec:universal}
Consider the multimeasurement (factorial kernel) generalization of the kernel density by~\citet{saremi2022multimeasurement} for $m$ isotropic Gaussian kernels with equal noise level  (kernel bandwidth) $\sigma$:
\<\label{eq:sigma-m-density} p(y_{1:m}) \propto \int_{\mathbb{R}^d} e^{-f(x)} \exp\Bigl(-\frac{1}{2\sigma^2} \sum_{t=1}^m \bv x - y_t\bv^2 \Bigr) dx.\>
We refer to $p(y_{1:m})$ as the \emph{$(\sigma,m)$-density}. Equivalently, $Y_t|x \sim \mathcal{N}(x,\sigma^2 I)$, $t \in [m]$ all independent. Clearly, $p(y_{1:m})$ is permutation invariant
$ p(y_1,\dots, y_m) = p(y_{\pi(1)},\dots,y_{\pi(m)}),$ where $\pi: [m] \rightarrow [m]$ is a permutation of the $m$ measurements. We set the stage for the remainder of the paper with a calculation that shows the permutation invariance takes the following form (see \autoref{sec:app:universal}):
\< \label{eq:log-sigma-m-density} \log p(y_{1:m}) = \varphi(\overline{y}_{1:m};m^{-1/2}\sigma) + \frac{m}{2\sigma^2}\Bigl(\Vert \overline{y}_{1:m} \Vert^2- \frac{1}{m} \sum_{t=1}^{m} \Vert y_t \Vert^2 \Bigr)+\cst,\>
where $$ \varphi(y; \sigma) \coloneqq \log \mathbb{E}_{x \sim \mathcal{N}(y,\sigma^2 I)}[e^{-f(x)}].$$

The calculation is straightforward by grouping the sums of squares in \eqref{eq:sigma-m-density}:
$$ -\sum_{t=1}^m \Vert x-y_t\Vert^2= m\Bigl(-\Vert x-\overline{y}_{1:m} \Vert^2+\Vert\overline{y}_{1:m}\Vert^2- \frac{1}{m} \sum_{t=1}^{m} \Vert y_t \Vert^2\Bigr)+\cst.$$
In addition, the Bayes estimator of $X$  given $Y_{1:m}=y_{1:m}$ simplifies as follows (see \autoref{sec:app:universal}):
\< \label{eq:xhat}
\mathbb{E}[X|y_{1:m}] = \overline{y}_{1:m} + m^{-1} \sigma^2  \nabla \varphi(\overline{y}_{1:m};m^{-1/2}\sigma) .
 \>
These calculations bring out a notion of universality class that is associated with $p(y_{1:m})$ formalized by the following definition and proposition.

\begin{definition}[Universality Class]\label{def:universality} We define the universality class $[\tilde{\sigma}]$ as the set of densities $p(y_{1:m})$, in the family of $(\sigma, m)$-densities,
such that for all $(\sigma,m) \in [\tilde{\sigma}]$ the following holds: $m^{-1/2} \sigma  =\tilde{\sigma}.$  
\end{definition}


\begin{proposition} \label{prop:universal} If $Y_{1:m} \sim p(y_{1:m})$, let $\hat{p}_{\sigma,m}$ be the distribution of $\mathbb{E}[X|Y_{1:m}]$, and define $\hat{p}_{\sigma}=\hat{p}_{\sigma,1}$. Then $\hat{p}_{\sigma,m}=\hat{p}_{m^{-1/2}\sigma}$. In other words, $\hat{p}_{\sigma,m}$ is identical for  densities in the same universality class.
\end{proposition}
\begin{proof}
We are given $X \sim e^{-f(x)}$, $Y_t = X + \varepsilon_t$, $\varepsilon_t \sim \mathcal{N}(0,\sigma^2 I)$ independently for $t\in [m]$. It follows $\overline{Y}_{1:m} = X + \tilde{\varepsilon}$, where $\tilde{\varepsilon} \sim \mathcal{N}(0,\tilde{\sigma}^2  I)$, where $\tilde{\sigma}^2=m^{-1}\sigma^2$.  Using \eqref{eq:xhat}, $\mathbb{E}[X|y_{1:m}]$ is distributed as 
 $$ X + \tilde{\varepsilon} + \tilde{\sigma}^2  \nabla \varphi(X + \tilde{\varepsilon}; \tilde{\sigma} ), $$
 which is identical for all densities $p(y_{1:m})$ in $[\tilde{\sigma}]$.
\end{proof}
\subsection{Distribution of $\mathbb{E}[X|y
_{1:m}]$ vs. $p_X$: upper bound on the 2-Wasserstein distance} \label{sec:wasserstein}
Our goal is to obtain samples from $p_X$, but in walk-jump sampling the samples are given by $\mathbb{E}[X|y_{1:m}]$, where $y_{1:m} \sim p(y_{1:m})$~\citep{saremi2019neural, saremi2022multimeasurement}. Next, we address how far  $\myhat{p}_{\sigma,m}$ is from the density of interest $p_X$.

\begin{proposition} \label{prop:p-vs-phat} The squared 2-Wasserstein distance between $p_X$ and $\myhat{p}_{\sigma,m}$ is bounded by
$$ W_2(p_X,\myhat{p}_{\sigma,m})^2 \myleq   \frac{\sigma^2}{m} d.$$
\end{proposition} 
The proof is given in \autoref{app:score-hessian}. As expected, the upper bound is expressed in terms of $\tilde{\sigma}^2 = \sigma^2 / m.$ A close inspection of the proof shows that the bound above is loose as it is obtained from the rate resulting from  replacing the empirical Bayes estimator $\mathbb{E}[X|y_{1:m}]$ with the empirical mean $\overline{Y}_{1:m}$. Note, however, that when the prior $p(x)$ is ``strong'' (e.g., low entropy), the dependence on $\sigma^2/m$ can be significantly improved.

\section{The geometry of $(\sigma,m)$-densities} \label{sec:geometry}
 In this section we analyze at the problem of sampling from $p(y_{1:m})$ where we consider $p(x)$ to be an anisotropic Gaussian, $X \sim \mathcal{N}(0,C)$, with a diagonal covariance matrix: \< \label{eq:elliptical-gaussian} C= \diag(\tau_1^2,\dots,\tau_d^2).\> The density $p_X$ is strongly log-concave with the property $\tau_{\max}^{-2} I \mypreceq \nabla^2 f(x) \mypreceq \tau_{\min}^{-2} I$, therefore its condition number is $\kappa=\tau_{\min}^{-2}\tau_{\max}^2$. Log-concave densities with $\kappa \gg 1$ are considered ill-conditioned. Since $Y_1 \sim \mathcal{N}(0, C+\sigma^2 I)$, the condition number for (single-measurement) smoothed density, which we denote by $\kappa_{\sigma,1}$ is given by:
%
\<\label{eq:kappa1} \kappa_{\sigma,1}=(1+\sigma^{-2} \tau_{\max}^2)/ (1+\sigma^{-2} \tau_{\min}^2).\> 
Next, we give the full spectrum of the precision matrix associated with $(\sigma,m)$-densities.
\begin{proposition} \label{prop:allatonce} Consider an anisotropic Gaussian density $X \sim \mathcal{N}(0,C)$ in $\mathbb{R}^d$, where $C_{ij} = \tau_i^2 \delta_{ij}$. Then the $(\sigma,m)$-density is a centered Gaussian in $\mathbb{R}^{md}$: $ Y_{1:m} \sim \mathcal{N}(0,F_{\sigma,m}^{-1})$. For $m \mygeq 2$, the precision matrix $F_{\sigma,m}$ is block diagonal with $d$ blocks (indexed by $i$) of size $m\times m$, each with the following spectrum: (i) There are $m-1$ degenerate eigenvalues equal to $\sigma^{-2},$ (ii)
The remaining eigenvalue equals to $(\sigma^2+m \tau_i^2)^{-1}.$ The condition number $\kappa_{\sigma,m}$ associated with the $(\sigma,m)$-density is given by:
$$ \kappa_{\sigma,m} =  \frac{\lambda_{\rm max}(F_{\sigma,m})}{\lambda_{\rm min}(F_{\sigma,m})}=1+m\cdot\sigma^{-2} \tau_{\rm max}^2. $$
\end{proposition}

%
\begin{remark}[The curse of sampling all measurements at once] \label{remark:dont}  The above proposition is a negative result regarding sampling from $p(y_{1:m})$ if\textemdash this is an important ``if''\textemdash all $m$ measurements $y_{1:m}$ are sampled in parallel (at the same time). This is because $m \sigma^{-2} = \tilde{\sigma}^{-2}$ remains constant for $m>1$ for $\sigmam \in [\tilde{\sigma}]$\textemdash even worse, the condition number $\kappa_{\sigma,m}$ is strictly greater than $\kappa_{\tilde{\sigma},1}$ for $m>1$.
\end{remark}

This negative result regarding the sampling scheme by~\cite{saremi2022multimeasurement}, we call \emph{all (measurements) at once} (AAO), leads to our investigation below into sampling from $p(y_{1:m})$  sequentially, \emph{one (measurement) at a time} (OAT), by accumulating measurements using the factorization  \eqref{eq:oneatatime}. Now, we perform the analysis in \autoref{prop:allatonce} for the spectrum of the conditional densities in \eqref{eq:oneatatime}.

\begin{proposition} \label{prop:oneatatime} Assume $X \sim \mathcal{N}(0,C)$ is the anisotropic Gaussian in Proposition~\ref{prop:allatonce}. Given the factorization of $p(y_{1:m})$ in \eqref{eq:oneatatime}, for $t > 1$, the conditional density $p(y_t|y_{1:t-1})$ is a Gaussian with a shifted mean, and with a diagonal covariance matrix:
$$ 
    -2\sigma^2 \log p(y_t|y_{1:t-1}) = \sum_{i=1}^d \left(1-\tsigmatau_{ti}\right) \cdot \Bigl(y_{ti}-\frac{\tsigmatau_{ti} }{1-\tsigmatau_{ti}} \sum_{k=1}^{t-1} y_{ki} \Bigr)^2   + \cst,
$$
 where $\tsigmatau_{ti}$ is short for
 $\tsigmatau_{ti} = \left(t+\sigma^2 \tau_i^{-2}\right)^{-1}$.
 The precision matrix  associated with  $p(y_t|y_{1:t-1})$, denoted by $F_{t|1:t-1}$, has the following spectrum
$$\sigma^2 \lambda_i(F_{t|1:t-1}) = 1-\left(t+\sigma^2 \tau_i^{-2}\right)^{-1},$$
with the following condition number $$\kappa_{t|1:t-1}= \frac{1 - (t+\sigma^2 \tau_{\min}^{-2})^{-1} }{1 - (t+\sigma^2 \tau_{\max}^{-2})^{-1}}. $$
Lastly, the condition number $\kappa_{t|1:t-1}$  is  monotonically decreasing as $t$ increases (for any $m>1$):
\< \label{eq:anisotropic-monotonicity} 1< \kappa_{m|1:m-1} < \cdots < \kappa_{3|1:2} < \kappa_{2|1} < \kappa_1,\>
where $\kappa_1 \coloneqq \kappa_{\sigma,1}$ is given by \eqref{eq:kappa1}.
\end{proposition}

The proofs for \autoref{prop:allatonce} and \autoref{prop:oneatatime} are given in \autoref{app:gaussian-kappa}.  These two propositions stand in a clear contrast to each other: in the OAT setting of  \autoref{prop:oneatatime}, sampling becomes \emph{easier} by increasing $t$ as one goes through accumulating measurements $y_{1:t}$ sequentially, where in addition $\kappa_1$ can itself be decreased by increasing $\sigma$. Next, we analyse the OAT scheme in more general settings.

\section{Chain of log-concave Markov chains} \label{sec:log-concave}
Can we devise a sampling scheme where we are guaranteed to always sample log-concave densities? This section is devoted to several results in that direction. We start with the following two lemmas.

\begin{lemma} \label{lemma:log-concave} Assume $\forall x \in \mathbb{R}^d$, $\nabla^2 f(x) \mypreceq L I$ and $\Vert \nabla f(x) \Vert \mygeq \mu \Vert x - x_0 \Vert -\Delta$ for some $x_0$. Then, $\forall y \in \mathbb{R}^d$:
$$ \nabla^2 (\log p)(y) \mypreceq \Bigl(-1 + \frac{3Ld}{\mu^2 \sigma^2}+\frac{3 \Delta^2}{\mu^2 \sigma^2}+3\frac{\Vert x_0-y\Vert^2}{\mu^2 \sigma^6}\Bigr) \frac{I}{\sigma^2}.$$
\end{lemma}
The proof is given in \autoref{app:logconcave}.


\begin{lemma}
\label{lemma:more-logconcave}
Consider the density $p(x)$ associated with the random variable $X$ in $\mathbb{R}^d$ and the $(\sigma,m)$-density given by \eqref{eq:sigma-m-density}. Then in expectation, for any $m \mygeq 1$ the conditional densities become more log-concave upon accumulation of measurements:\footnote{Note that here no assumption is made on the smoothness of $p(x)$.}
$$ \mathbb{E}_{y_1} \nabla^2_{y_1} \log p(y_1) \mysucceq \mathbb{E}_{y_{1:2}}  \nabla^2_{y_2} \log p(y_2|y_1) \mysucceq \cdots  \mysucceq \mathbb{E}_{y_{1:m}} \nabla^2_{y_m} \log p(y_m|y_{1:m-1}). $$
\end{lemma}

\begin{proof} The full proof of the lemma is given in \autoref{app:logconcave}, where we derive the following: $$
    \nabla^2_{y_m} \log p(y_m|y_{1:m-1}) =   - \sigma^{-2} I + \sigma^{-4} \cov(X|y_{1:m}).
$$
The proof follows through since due to the law of total covariance the mean of the posterior covariance $\mathbb{E}_{y_{1:m}}\cov(X|y_{1:m})$ can only go down upon accumulation of measurements.
\end{proof}

These two lemmas paint an intuitive picture that we expand on in the remainder of this section: (i) by increasing $\sigma$ we can transform a density to be strongly log-concave (\autoref{lemma:log-concave}) which we can sample our first measurement from, (ii) and by accumulation of measurements we expect sampling to become easier, where in \autoref{lemma:more-logconcave} this is formalized by showing that on average the conditional densities become more log-concave by conditioning on previous measurements. Next, we generalize these results with our main theorem, followed by an example on a mixture of Gaussians.
\begin{theorem}[Once log-concave, always log-concave] \label{theorem:compact} Consider $Z$ to be a random variable in $\mathbb{R}^d$ with a compact support, i.e., almost surely $\Vert Z \Vert^2 \myleq R^2$, and take $X=Z+N_0$, $N_0 \sim \mathcal{N}(0,\tau^2 I)$.  Then, for any $m\mygeq 1$, the conditional Hessian is upper bounded \<
\label{eq:compact:upper-bound} \nabla^2_{y_m} \log p(y_m|y_{1:m-1}) \mypreceq \zeta(m) I,\> where:
\< \label{eq:compact:zeta} \zeta(m) = \frac{1}{\sigma^2}\Bigl(\frac{\tau^2}{m \tau^2+\sigma^2}-1\Bigr)+\frac{R^2}{(m \tau^2+\sigma^2)^2}\>
is a decreasing function of $m$, in particular:
$$\zeta'(m) = -\frac{\tau^2(2 R^2 \sigma^2+\sigma^2 \tau^2+m\tau^4)}{\sigma^2(\sigma^2+m\tau^2)^3} \myleq 0.$$
As a corollary, $p(y_1)$ associated with  $Y_1=X+N_1$, $N_1 \sim \mathcal{N}(0,\sigma^2 I)$ is strongly log-concave if \< \label{eq:compact:log-concave} \sigma^2> R^2-\tau^2,\>
and stays strongly log-concave upon accumulation of measurements.
\end{theorem}
\begin{proof} The full proof is given in \autoref{app:logconcave} and it is a direct consequence of the following identity:
$$ \nabla^2_{y_m} \log p (y_m|y_{1:m-1}) = \frac{1}{\sigma^2}\Bigl(\frac{\tau^2}{m\tau^2+\sigma^2}-1\Bigr) \cdot I + \frac{1}{(m\tau^2+\sigma^2)^2} \cov(Z|y_{1:m}), $$
which we derive, combined with $\cov(Z|y_{1:m}) \mypreceq R^2 I$ due to our compactness assumption. 
\end{proof}

\begin{remark} \autoref{theorem:compact} spans a broad class of sampling problems, especially since $\tau$ can in principle be set to zero. The only property we loose in the setting of $\tau=0$ is that the upper bound $\zeta(m) I$ does not monotonically go down  as measurements are accumulated. 
\end{remark}


\begin{figure}[t!]
\begin{center}
\begin{subfigure}[$\sigma=2$]
{\includegraphics[width=0.33\textwidth]{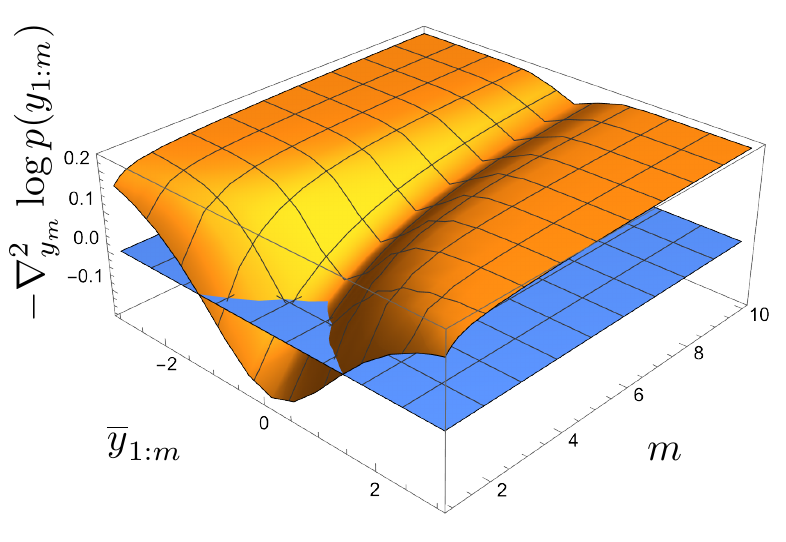}}    
\end{subfigure}
\begin{subfigure}[$\sigma=2\sqrt{2}$]
{\includegraphics[width=0.33\textwidth]{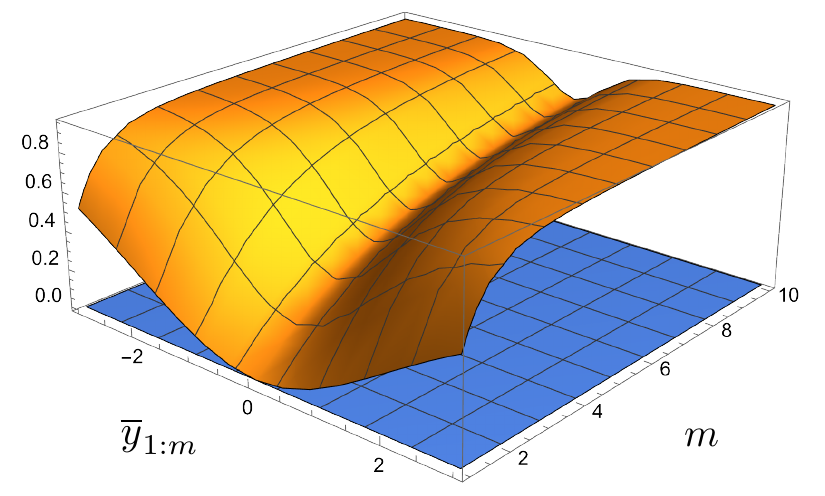}}  
\end{subfigure}
\begin{subfigure}[$\sigma^2 \zeta(m)$ vs. $m$]
{\includegraphics[width=0.31\textwidth]{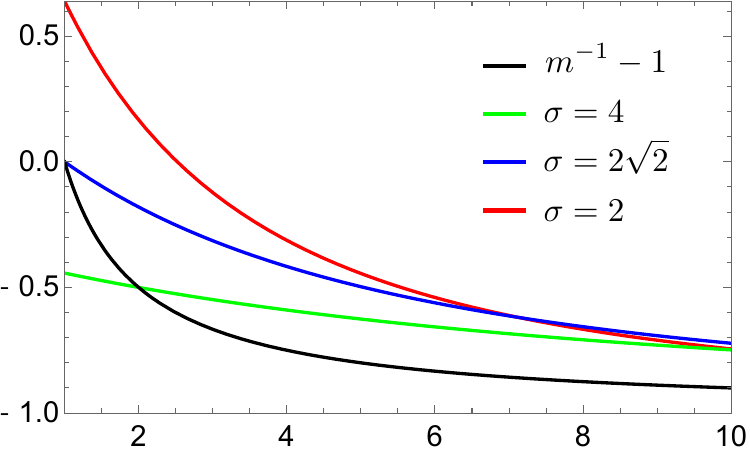}}   
\end{subfigure}
\end{center}
\vspace*{-.2cm}

\caption{\label{fig:hessian} (a,b) The negative conditional Hessian  for two values of $\sigma$ are plotted as a function of $\overline{y}_{1:m}$ and $m$ assuming $X$ is distributed according to \eqref{eq:two-mixture} in 1D, where we set $\mu=3$, $\tau=1$ (see \eqref{eq:hess-general}).  (c) The upper bound in \eqref{eq:compact:upper-bound} is sharp for this example; $\sigma^2 \zeta(m)$ is plotted vs. $m$ for different $\sigma$. }
\end{figure}

\subsection{Example: Mixture of two Gaussians} \label{example:mog}
In this section we examine \autoref{theorem:compact} by studying the following mixture of Gaussians for $\alpha=1/2$:
\< \label{eq:two-mixture}
p(x) = \alpha\,\mathcal{N}(x;\mu,\tau^2 I)+ (1-\alpha)\, \mathcal{N}(x;-\mu,\tau^2 I).
\>
This is an instance of the setup in \autoref{theorem:compact}, where $p(z) = \frac{1}{2} \delta(z-\mu)+ \frac{1}{2} \delta(z+\mu)$, and $R^2 = \mu^\top \mu$.
By differentiating \eqref{eq:log-sigma-m-density} twice we arrive at  the following expression for $\nabla_{y_m}^2 \log p(y_m|y_{1:m-1})$:
\< \label{eq:hess-general}
 \nabla_{y_m}^2 \log p(y_m|y_{1:m-1}) =  \nabla_{y_m}^2 \log p(y_{1:m}) = \sigma^{-2} (m^{-1}-1) I + m^{-2} H(\overline{y}_{1:m}; m^{-1/2} \sigma),\>
where $$ H(y;\sigma) \coloneqq \nabla^2 \log \mathbb{E}_{X \sim \mathcal{N}(y,\sigma^2 I)}[e^{-f(X)}].$$
In \autoref{app:logconcave} we show that for the mixture of Gaussian here, \eqref{eq:two-mixture} with $\alpha=1/2$, we have
\< \label{eq:mog-hess}
\hess(y;\sigma) = \frac{1}{(\sigma^2+\tau^2)} 
  \biggl(-I + \frac{2 \mu \mu^{\top}}{\sigma^2+\tau^2} \cdot \biggl(1+\cosh\Bigl(\frac{2\mu^{\top} y}{\sigma^2+\tau^2}\Bigr)\biggr)^{-1} \biggr), \>
which takes its maximum at $y=0$. By using \eqref{eq:hess-general}, it is then straightforward to show that \eqref{eq:compact:upper-bound}, \eqref{eq:compact:zeta}, and \eqref{eq:compact:log-concave} all hold in this example, with the additional result that the upper bound is now tight. In \autoref{fig:hessian}, these calculations are visualized in 1D for $\mu=3, \tau=1$, and for different values of $\sigma$; in panel (c) we also plot $1-1/m$ which is the large $m$ behavior of $\sigma^2 \zeta(m)$. This can be seen from two different routes: \eqref{eq:compact:zeta} and \eqref{eq:hess-general}.

\begin{remark}[Monotonicity] The monotonic decrease of the upper bound in \autoref{theorem:compact}, together with the monotonicity result in \autoref{lemma:more-logconcave}, may lead one to investigate whether the stronger result
\< \label{eq:monotonicity} \nabla_{y_1}^2 \log p(y_1) \mysucceq \nabla_{y_2}^2 \log p(y_2|y_1) \mysucceq \dots \mysucceq \nabla_{y_m}^2 \log p(y_m|y_{1:m-1}),\>
could hold, e.g., for the mixture of Gaussians we studied here, especially since the upper bound \eqref{eq:compact:upper-bound} is sharp for this example. For \eqref{eq:monotonicity} to hold, $\cov(Z|y_{1:m})$ must be less than $\cov(Z|y_{1:m-1})$. However, we can imagine a scenario where $y_1+\dots+y_{m-1}$ is very large, so that $\cov(Z|y_{1:m}) \approx 0$, while $y_{m}$ is such that $y_1+\dots+y_{m-1}+y_m$ is close to $m \mathbb{E}[Z]$, where $\cov(Z|y_{1:m})$ will be large. 
\end{remark}

\subsection{Algorithm: non-Markovian chain of (log-concave) Markov chains}
Below, we give the pseudo-code for our sampling algorithm. In the inner loop, $\mcmc_\sigma$ is any MCMC method, but our focus in this paper is on Langevin MCMC algorithms\footnote{We experimented with a variety of Langevin MCMC algorithms to sample from $p(y_t|y_{1:t-1})$ in the inner loop of Algorithm~\ref{alg:oat}. The results are reported in the appendix due to space constraints. After extensive tuning, we found the underdamped Langevin MCMC algorithm by~\citet{sachs2017langevin} to be the best performing for the test densities we considered.} that use $\nabla_{y_t} \log p(y_t|y_{1:t-1})$ to sample the new measurement $Y_t$ conditioned on the previously sampled ones $Y_{1:t-1}$.

\label{sec:algorithm}
\begin{algorithm}[H]
\caption{{One-(measurement)-at-a-time} walk-jump sampling referred to by OAT. See \autoref{fig:schematic} for the schematic. A version of $\mcmc_\sigma$  is given in  \autoref{app:detailed_algorithm}.}
 \begin{algorithmic}[1]
   \STATE \textbf{Parameter} noise level $\sigma$
   \STATE \textbf{Input} number of measurements $m$, number of steps for each measurement $n_t$ 
   \STATE \textbf{Output} $\myhat{X}$
   \STATE {Initialize ${\overline Y}_{1:0} = 0$ } \\
   \FOR{$t=[1,\dots,m]$} 
   \STATE {Initialize $Y_t^{(0)}$} \\
   
   \FOR{$i=[1,\dots,n_t]$}
    \STATE $Y_t^{(i)} = {\textrm{MCMC}_\sigma} (Y_t^{(i-1)}, {\overline Y_{1:t-1}})$ \\
     \ENDFOR \\
     \STATE $Y_t = Y_t^{(n_t)}$
     \STATE ${\overline Y}_{1:t} = {\overline Y}_{1:t-1} + (Y_t - {\overline Y}_{1:t-1})/t$    \\  
   \ENDFOR
   \RETURN  $\myhat{X} \leftarrow \mathbb{E}[X|Y_{1:m}]$  according to \eqref{eq:xhat} 
 \end{algorithmic}
 \label{alg:oat}
\end{algorithm}

\subsubsection{Estimating $\nabla \log p(y)$}
So far we have assumed we know the smoothed score function $g(y;\sigma)=\nabla (\log p)(y)=\nabla \varphi(y;\sigma)$, and in experiments below we consider cases where we know $g(y;\sigma)$ in closed form. But in principle, we would like to estimate $g(y;\sigma)$ in terms of the unnormalized $p(x) \propto e^{-f(x)}$. The simplest such ``plug-in'' estimator  is as follows (see \autoref{app:score-hessian}):
\< \label{eq:ghat} \myhat{g}_1(y; \sigma) = \frac{1}{\sigma}\frac{\sum_{i=1}^n \varepsilon_i \exp\left(-f(y+\sigma \varepsilon_i)\right)}{\sum_{i=1}^n \exp\left(-f(y+\sigma \varepsilon_i)\right)},\; \varepsilon_i \sim \mathcal{N}(0,I).\>
Using any estimator for $\nabla (\log p)(y)$, including $\hat{g}_1$ above, we can construct an estimator for $\nabla \log p(y_m|y_{1:m-1})$ using results in \autoref{sec:universal}, see \eqref{eq:log-sigma-m-density}, and \autoref{sec:app:universal}:
\<\label{eq:ghat_oat} \nabla_{y_m} \log p(y_m|y_{1:m-1}) = \nabla_{y_m} \log p(y_{1:m}) \approx \frac{1}{m} \myhat{g}(\overline{y}_{1:m};\frac{\sigma}{\sqrt{m}})+\frac{1}{\sigma^2}(\overline{y}_{1:m}-y_t).  \> 
In the appendix, we also conduct experiments to investigate this aspect of the problem. Studying the covariance of the plug-in estimator and devising better estimators is beyond the scope of this paper.

\section{Experiments} \label{sec:experiments}

We evaluate the performance of Algorithm~\ref{alg:oat} alongside related sampling schemes on carefully designed test densities. We compare the following sampling schemes:
\begin{itemize}
    \item One-at-a-time walk-jump sampling using  Algorithm~\ref{alg:oat} with $m=1000$ (``OAT''),
    \item All-at-once walk-jump sampling  (``AAO''), 
    \item Single-measurement walk-jump sampling (``$m=1$''),
    \item Langevin MCMC by~\cite{sachs2017langevin} without any smoothing (``$\sigma=0$'').
\end{itemize}
The hyperparameters were tuned for each sampling scheme and the total number of iterations was kept fixed. See \autoref{app:experimental_detail} for details.

\begin{figure}[t!]
\begin{center}
\begin{subfigure}[]
    {\includegraphics[width=0.24\textwidth]{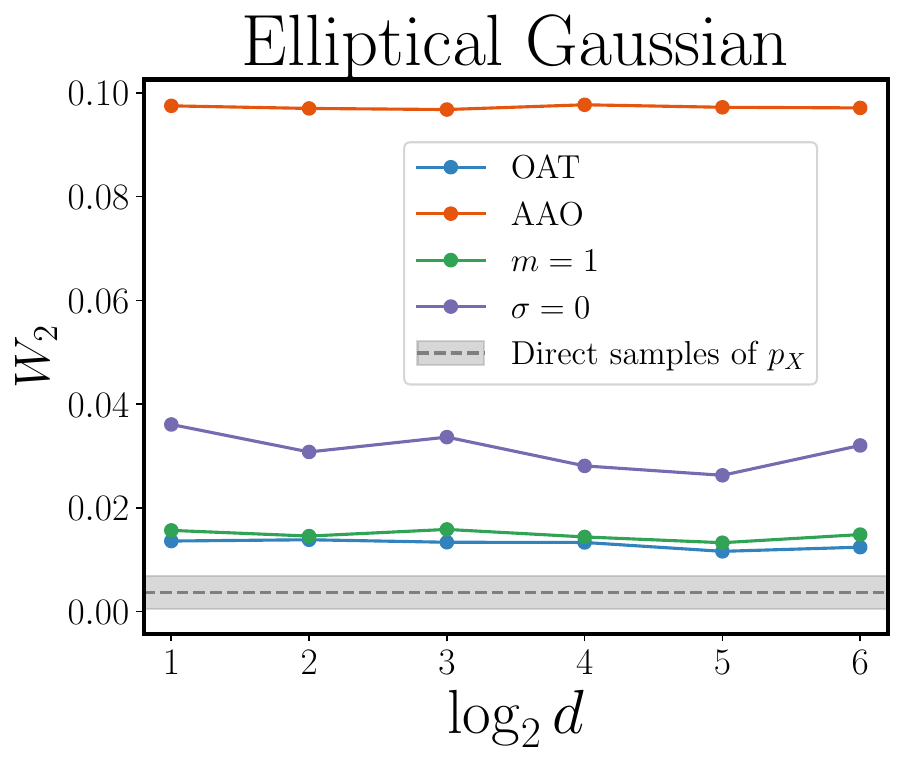}} 
\end{subfigure}
\begin{subfigure}[]
    {\includegraphics[width=0.24\textwidth]{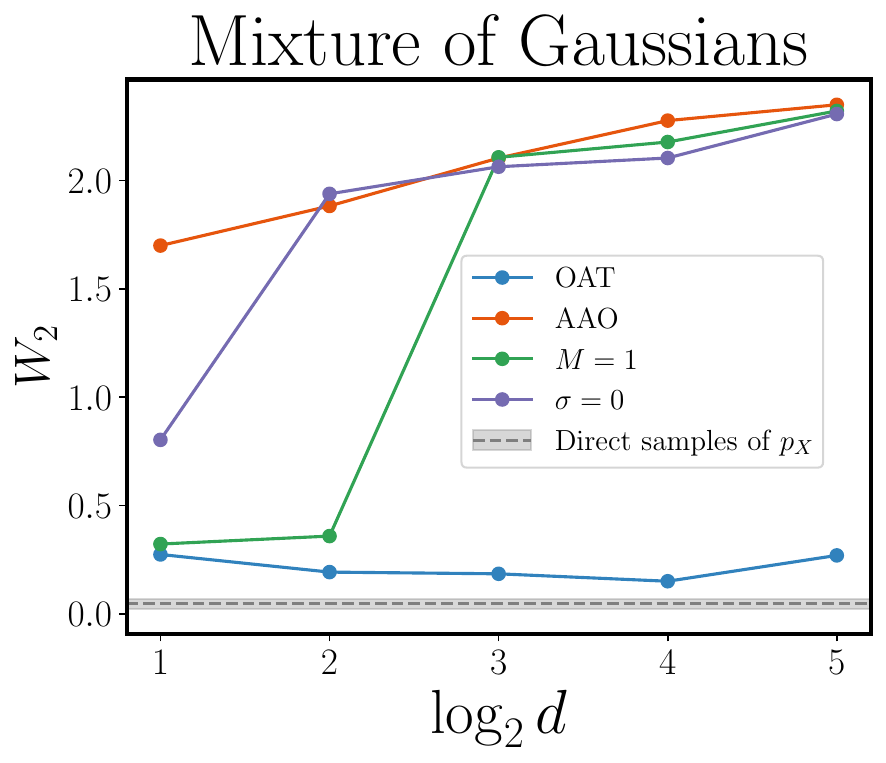}} 
\end{subfigure}
\begin{subfigure}[]
    {\includegraphics[width=0.24\textwidth]{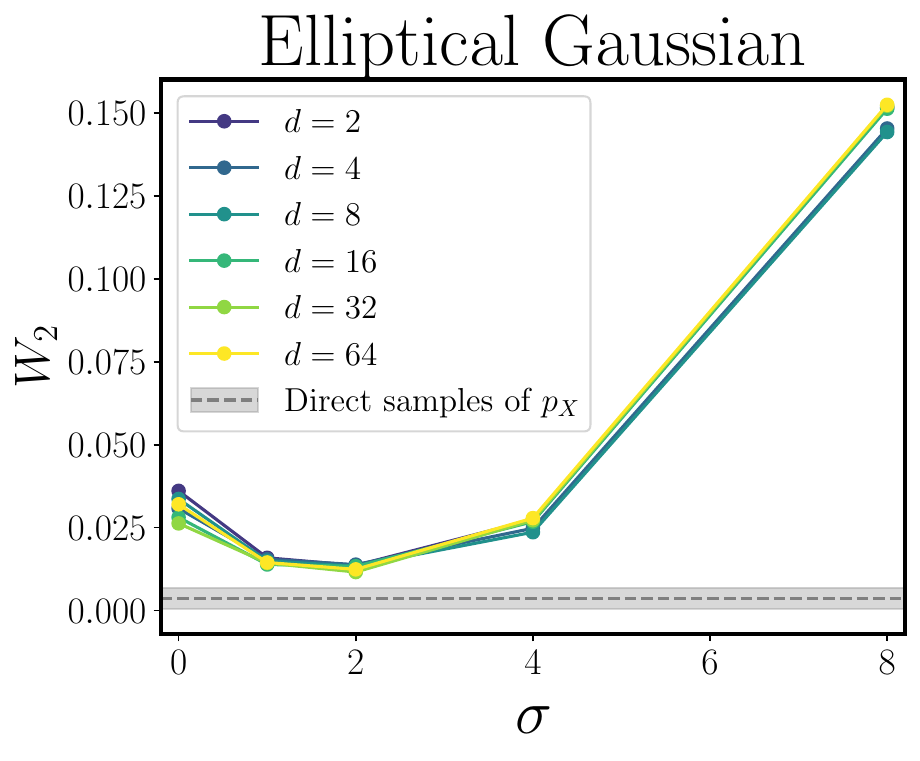}} 
\end{subfigure}
\begin{subfigure}[]
    {\includegraphics[width=0.24\textwidth]{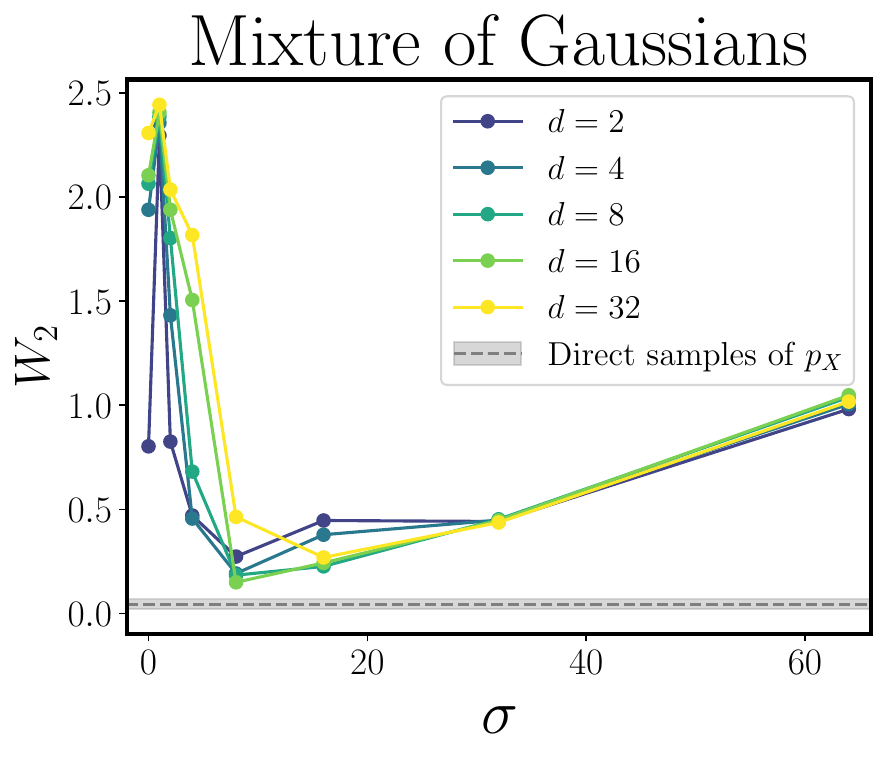}} 
\end{subfigure}
\end{center}
\vspace*{-0.4cm}
\caption{\label{fig:elliptical_vs_dim} \label{fig:elliptical_vs_sigma} \label{fig:gmm_vs_dim} 
\label{fig:gmm_vs_sigma} (a, b) 2-Wasserstein distance vs. $d$ and (c, d) 2-Wasserstein distance vs. $\sigma$ for varying $d$ for the elliptical Gaussian and Gaussian mixture target densities.} 
\end{figure}
%
{\bf Metric.} We use the Wasserstein metric to quantify the consistency of the obtained samples with the target density $p_X$. In general, evaluation of Wasserstein distance between high-dimensional distributions is nontrivial, as it involves identifying an optimal coupling and solving a multivariate integral. We implement a simplified version of this metric by viewing our target measure as empirical and restricting our focus to a representative marginal dimension by projecting our samples $X$ from $p_X$ and $\myhat{X}$ from ${\myhat p}_{\sigma, m}$ to a chosen vector $\axis$ in $\mathbb{R}^d$: $X^{\parallel} = \axis^\top X \in \mathbb{R}$. For one-dimensional empirical distributions, the $2$-Wasserstein distance between $p_X$ and ${\myhat p}_{\sigma, m}$ can be approximated as
$$
    W_2(X^\parallel, \myhat{X}^\parallel)^2 \approx \frac{1}{n} \sum_{i=1}^n |X^\parallel_{(i)} - \myhat{X}^\parallel
    _{(i)} |^2 ,
$$
where $X^\parallel_{(1)}, \cdots, X^\parallel_{(n)}$  are order statistics~\citep{bonneel2015sliced, peyre2019computational}. 


{\bf MCMC algorithms.}  For all the results in this section, we implement MCMC sampling based on underdamped Langevin diffusion (ULD). The particular algorithm used for the results shown in this section extends the BAOAB integration scheme using multiple time steps for the O-part \citep{sachs2017langevin}. In \autoref{app:mcmc_alg}, we present the full comparison across other MCMC algorithms, including other recent ULD variants~\citep{cheng2018underdamped, shen2019randomized} as well as the Metropolis-adjusted Langevin algorithm (MALA)~\citep{roberts1996exponential, dwivedi2018log}. 

{\bf Score estimation.}  In \autoref{app:score}, we compare sampling with the analytic score function and the plug-in estimator of the score function given in \eqref{eq:ghat} with varying numbers of MC samples $n$. 

\subsection{Elliptical Gaussian} \label{sec:elliptical_gaussian}

The elliptical Gaussian features a poorly conditioned covariance: $p_X(x) = \mathcal{N} \left(x; \ 0, C\right),$
where we set $\tau_0^2 = 0.1, \tau_1^2 = \cdots = \tau_d^2 = 1$ in \eqref{eq:elliptical-gaussian}. We evaluate the 2-Wasserstein distance on the ``difficult'' narrow dimension (with variance $\tau_0^2$). For each $d$, the noise level $\sigma$ and other hyperparameters of the sampling algorithm, such as step size and friction, were tuned. \autoref{fig:elliptical_vs_dim}(a) plots the 2-Wasserstein distance with varying $d$. OAT and $m=1$ outperform Langevin MCMC ($\sigma=0$) for all $d$, suggesting that smoothing offers an advantage. On the other hand, AAO struggles, which is expected from our theoretical analysis in \autoref{sec:geometry}. As \autoref{fig:elliptical_vs_sigma}(b) shows, OAT with $\sigma=1$ and $\sigma=2$ outperforms Langevin MCMC ($\sigma=0$) for all $d$. The optimal $\sigma$ remains fairly constant with increasing $d$.  
\subsection{Mixture of Gaussians} \label{sec:gmm}
To evaluate mixing of multiple modes, we consider the mixture of Gaussians test density \eqref{eq:two-mixture} with $\alpha=\nicefrac{1}{5}$, $\tau=1$, and $\mu = 3 \cdot 1_d$, where $1_d$ is the $d$-dimensional vector $(1,\dots,1)^\top$. As \autoref{fig:gmm_vs_dim}(a) shows, OAT achieves consistently low 2-Wasserstein distance with increasing $d$, whereas other sampling schemes deteriorate in performance. We observe, in \autoref{fig:gmm_vs_sigma}(b), that OAT outperforms the best-performing underdamped Langevin MCMC ($\sigma=0$) in our experiments  for at least one $\sigma$ value for all $d$. Higher $d$ requires larger $\sigma$. In addition, we would like to highlight the following: {\bf (I)} Vanilla walk-jump sampling ($m=1$) is highly ineffective as the dimensions increase. This is in contrast to the sampling from anisotropic Gaussian (already log-concave) in \autoref{sec:elliptical_gaussian}. {\bf (II)} The optimal $\sigma$ here is in fact larger than the noise level needed to make  $p(y_1)$ log-concave. This is related to the  benefits of sampling from better-conditioned log-concave distributions, which is well-known in the literature. In \autoref{app:gmm_correlated}, we  include results for a mixture of correlated Gaussians supporting the same conclusion.



\subsection{Tunneling phenomena }
\autoref{fig:trajectory} illustrates the trajectories of three walkers (a) under our OAT sampling scheme and (b) using Langevin MCMC. Each walker has the same random seed between (a) and (b) and was initialized at $(3, 3)$, the dominant mode with 80\% of the mass. With OAT, a walker is able to tunnel to the smaller mode fairly quickly, whereas for Langevin MCMC (without smoothing) all three walkers are stuck around the dominant mode. In panels (c) and (d) we also show the histogram of final samples in the same setup with initialization at $(3, 3)$ for 100 walkers after 100\,K steps.

In summary of this section, we consistently observe that the same (Langevin) MCMC algorithm, when used in the inner loop of Algorithm~\ref{alg:oat}, is more effective than when used without Gaussian smoothing. 


\begin{figure}[t!]
\begin{center}
\begin{subfigure}[OAT, 3 trajectories]    {\includegraphics[ width=0.24\textwidth]{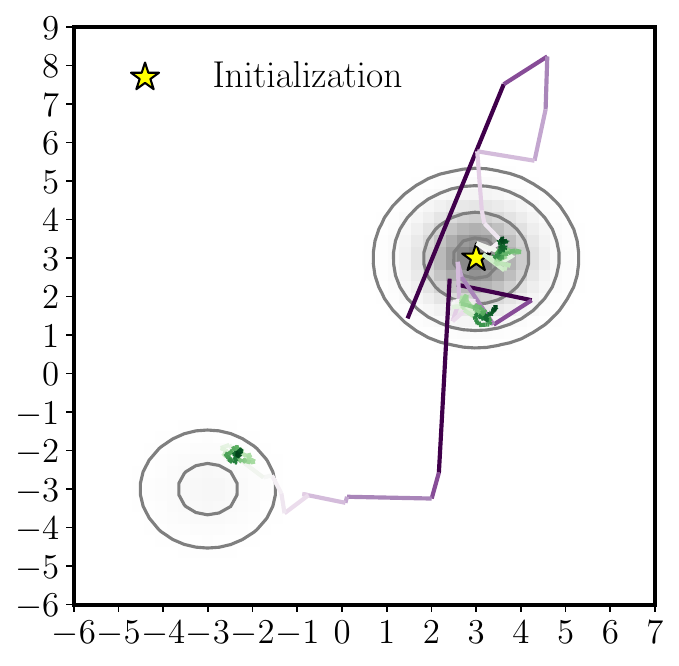}}  
\end{subfigure}
\begin{subfigure}[Langevin MCMC]{\includegraphics[ width=0.24\textwidth]{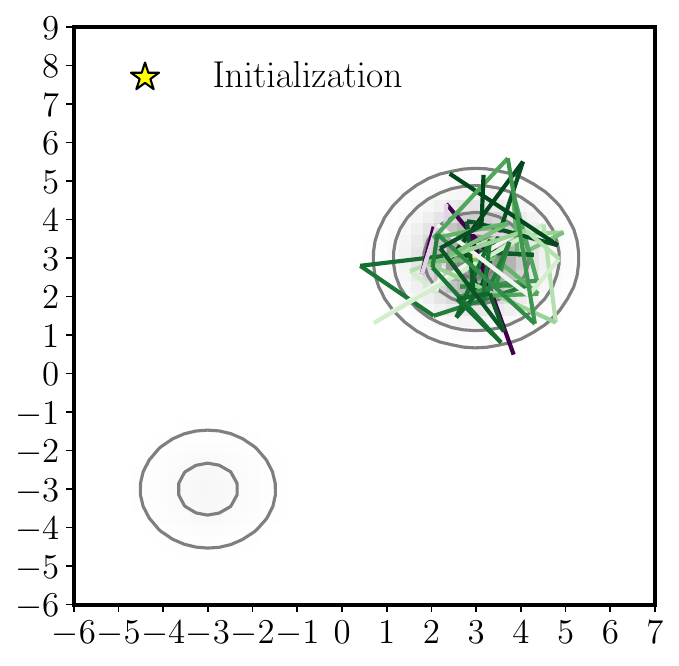}}  
\end{subfigure}
\begin{subfigure}[OAT, 100 samples]
{\includegraphics[width=0.24\textwidth]{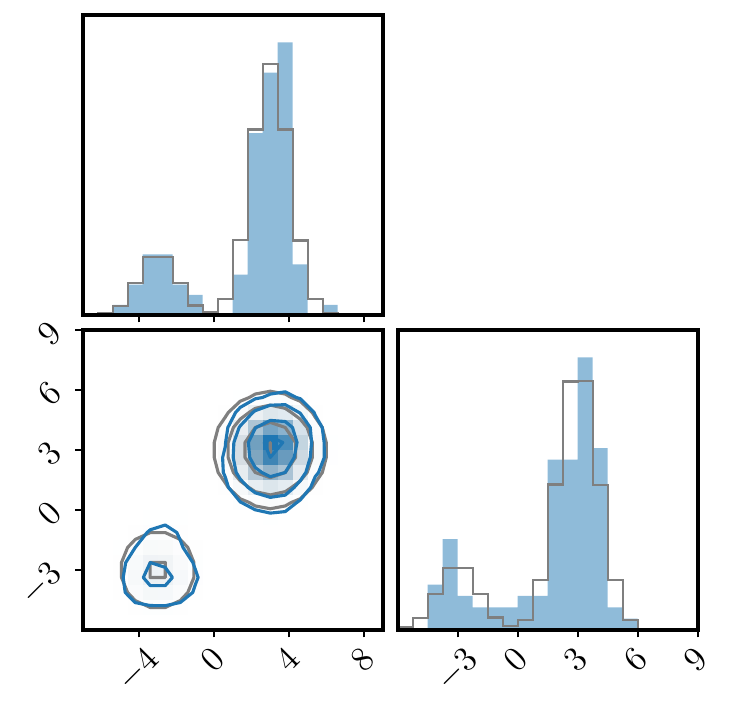}}  
\end{subfigure}
\begin{subfigure}[Langevin, 100 samples]
{\includegraphics[width=0.24\textwidth]{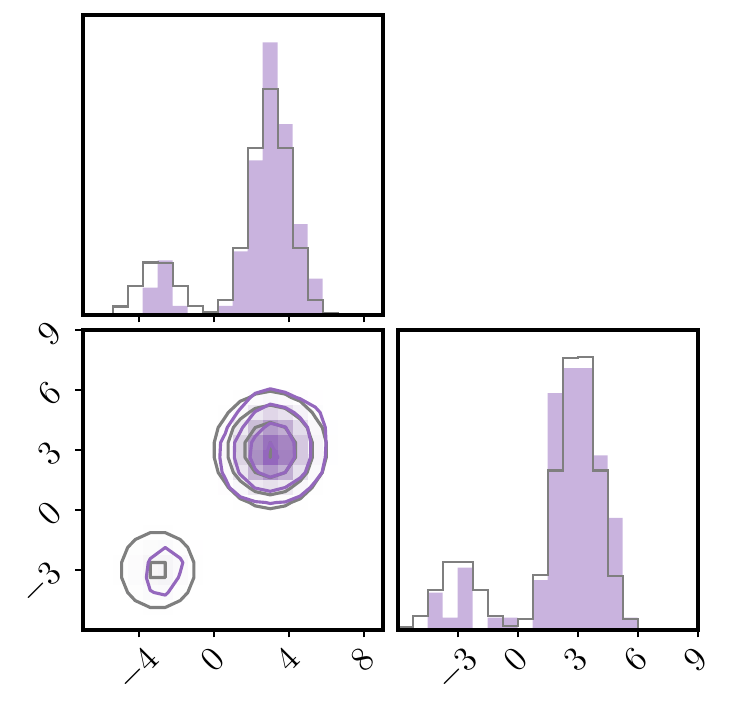}}   
\end{subfigure}
\end{center}
\vspace*{-0.3cm}
\caption{\label{fig:trajectory} {Tunneling phenomenon.} Trajectories of three walkers (a) under our OAT sampling scheme using Langevin MCMC by~\citet{sachs2017langevin} in the inner loop of Algorithm~\ref{alg:oat} (b) the same Langevin MCMC without any smoothing. Purple and green indicate the beginning and end of trajectories, respectively. (c, d) The final samples for 100 random trajectories (using identical seeds).}
\end{figure}


\section{Conclusion} \label{sec:conclusion}
In this paper, we established a theoretical framework that reduces the general problem of sampling from an unnormalized distribution to that of log-concave sampling defined by a single noise parameter.
We conclude with two main limitations of this work at the present time: {\bf (I)} Our results do not make it clear if the log-concave sampling strategy is optimal. The issue of ``optimality'' is challenging as it is inherently problem-dependent and additionally depends on the MCMC algorithm used in the inner loop of Algorithm~\ref{alg:oat}. {\bf (II)} Related to the issue of optimality is the fact that for general sampling problems, the smoothed score functions required to sample from $p(y_t|y_{1:t-1})$, $t \in[m]$ using gradient-based methods need to be estimated. This is a complex problem and should be investigated in future research. Finally, an immediate application of the machinery we developed here is the problem of generative modeling, as the $m$ smoothed score functions needed in running Algorithm~\ref{alg:oat} can indeed be learned (approximated) using empirical Bayes least-squares denoising objectives.
This approach is similar to training diffusion models; however, our sampling scheme is fundamentally different, as it relies on the accumulation of measurements, controlled by a single noise parameter, $\sigma$.

\bibliography{biblio.bib}
\bibliographystyle{iclr}


\clearpage

\appendix

\section{Derivations for \autoref{sec:universal} }
\label{sec:app:universal}
\subsection{Derivation for \eqref{eq:log-sigma-m-density}}
In this section we derive a general  expression for the $\sigmam$-densities $p(y_{1:m})$, short for $p_\sigma(y_{1:m})$:
\< \nonumber \begin{split} p(y_{1:m}) &= \int_{\mathbb{R}^d} p(x) \biggl( \prod_{t=1}^m p(y_t|x) \biggr) dx \\ &= \int_{\mathbb{R}^d} \frac{1}{Z}e^{-f(x)} \biggl( \prod_{t=1}^m \frac{1}{(2\pi \sigma^2)^{d/2}}\exp\Bigl(-\frac{1}{2\sigma^2} \bv x - y_t\bv^2 \Bigr) \biggr)  dx \\
&\propto \int_{\mathbb{R}^d} e^{-f(x)} \exp\Bigl(-\frac{1}{2\sigma^2} \sum_{t=1}^m \bv x - y_t\bv^2 \Bigr) dx.
 \end{split}
\>
The $\sigmam$-density is permutation invariant under the permutation $\pi: [m] \rightarrow [m]$ of measurement indices. In the calculation below we derive a general form for $\sigmam$-densities where this permutation invariance becomes apparent in terms of the empirical mean of the $m$ measurements
$$ \overline{y}_{1:m} =  \frac{1}{m} \sum_{t=1}^m y_t. $$
and the empirical mean of $\{\Vert y_t \Vert^2\}_{t=1}^m$.  We start with a rewriting of $\log p(y_{1:m}|x)$:
\< \nonumber
 \begin{split}
    - 2 \sigma^2 \log p(y_{1:m}|x) 
    &= \sum_{m=1}^M \bv y_m - x \Vert ^2 + \cst \\  
    &= m \Vert x \Vert^2 -2 \langle \sum_{t=1}^m y_t, x\rangle + \sum_{t=1}^m \Vert y_t \Vert^2 + \cst \\
    &= m\Bigl(\Vert x - \overline{y}_{1:m}\Vert^2  -\Vert \overline{y}_{1:m} \Vert^2 + \frac{1}{m} \sum_{t=1}^m \Vert y_t\Vert^2\Bigr)+ \cst,
 \end{split}
 \>
 where $\cst$  is a constant that does not depend on $y_{1:m}$. Now, we view the smoothing kernel $p(y_{1:m}|x)$ as a Gaussian distribution centered at $\overline{y}$:
 \< \nonumber \begin{split}
    & \log p(y_{1:m}) = \log \int e^{-f(x)} p(y_{1:m}|x)\,dx +\cst  \\
    &= \log \int e^{-f(x)} \mathcal{N}(x; \overline{y}, m^{-1} \sigma^2 I_d)\,  dx + \frac{\Vert \overline{y}_{1:m} \Vert^2- m^{-1} \sum_{t=1}^m \Vert y_t\Vert^2}{2m^{-1}\sigma^2} + \cst  \\
    &= \log \mathbb{E}_{x\sim \mathcal{N}(\overline{y}_{1:m},m^{-1} \sigma^2 I)} [e^{-f(x)}]+\frac{\Vert \overline{y}_{1:m} \Vert^2- m^{-1} \sum_{t=1}^m\Vert y_t\Vert^2}{2m^{-1}\sigma^2}+\cst.
 \end{split}
   \>
The equation above reduces to \eqref{eq:log-sigma-m-density} with the following definition  $$\varphi(y;\sigma) \coloneqq \log \mathbb{E}_{x \sim \mathcal{N}(y,\sigma^2 I)}[e^{-f(x)}].$$
\subsection{Derivation for \eqref{eq:xhat}}
Next, we derive the expression for $\myhat{x}(y_{1:m})= \mathbb{E}[X|y_{1:m}]$ given in \eqref{eq:xhat}:
\< \label{eq:app:xhat}
\begin{split}
    \mathbb{E}[X|y_{1:m}] &= y_t + \sigma^2 \nabla_{y_t} \log p(y_{1:m}) \\
    &= y_t + \sigma^2 \bigl(m^{-1} \nabla \varphi(\overline{y}_{1:m};m^{-1/2}\sigma)+\sigma^{-2}(\overline{y}_{1:m}-y_t)\bigr) \\
    &= \overline{y}_{1:m} + m^{-1}\sigma^2 \nabla \varphi(\overline{y}_{1:m};m^{-1/2}\sigma). 
\end{split}
\>
The first equation above comes from the generalization of the Bayes estimator to factorial kernels~\citep{saremi2022multimeasurement}, and for the second equation we used \eqref{eq:log-sigma-m-density}.

\section{Smoothed score functions} \label{app:score-hessian}
In this section, we give several different  expressions related to smoothed densities used in the paper. We consider $Y \sim \mathcal{N}(X,\sigma^2 I)$ where $X \sim e^{-f(x)}$ in $\mathbb{R}^d$, thus with the density
$$ p(y) = \int \frac{1}{Z} e^{-f(x)} \frac{1}{(2\pi \sigma^2)^{d/2}} e^{-\frac{1}{2\sigma^2}\Vert x-y\Vert^2} dx \propto \int e^{-f(x)}  e^{-\frac{1}{2\sigma^2}\Vert x-y\Vert^2} dx. $$
We have the following expressions for $\log p(y)$:
\begin{align*} 
 \log p(y) &= \log \int e^{-f(x)}  e^{-\frac{1}{2\sigma^2}\Vert x-y\Vert^2} dx + \cst,  \\
  \log p(y) &= \log \int e^{-f(x+y)}  e^{-\frac{1}{2\sigma^2}\Vert x\Vert^2} dx + \cst  ,
\end{align*}
leading to
\begin{align*} 
\nabla (\log p) (y) &= \frac{\int e^{-f(x)}  e^{-\frac{1}{2\sigma^2}\Vert x-y\Vert^2} \frac{x-y}{\sigma^2} dx}{\int e^{-f(x)}  e^{-\frac{1}{2\sigma^2}\Vert x-y\Vert^2} dx} = \frac{1}{\sigma^2}(\mathbb{E}[X|y]-y),\\  
\nabla (\log p) (y) &= - \frac{\int e^{-f(x+y)} \nabla f(x+y) e^{-\frac{1}{2\sigma^2}\Vert x\Vert^2} dx}{\int e^{-f(x+y)}  e^{-\frac{1}{2\sigma^2}\Vert x\Vert^2} dx} \\
&= - \frac{\int e^{-f(x)} \nabla f(x) e^{-\frac{1}{2\sigma^2}\Vert x-y\Vert^2} dx}{\int e^{-f(x)}  e^{-\frac{1}{2\sigma^2}\Vert x-y\Vert^2} dx} = - \mathbb{E}[\nabla f(X)|y].
\end{align*}
This in turn leads to three expressions for the Hessian:
\begin{align}
\nonumber
    \nabla^2 (\log p) (y) &= -\frac{1}{\sigma^2} I + \frac{1}{\sigma^4} \left(\mathbb{E}\left[x (x-y)^\top|y\right]-\mathbb{E}\left[x|y\right] \mathbb{E}\left[x-y|y\right]^\top\right) \\
\label{app:eq:hessian-1}    
    &= -\frac{1}{\sigma^2} I + \frac{1}{\sigma^4}\left(\mathbb{E}\left[x x^\top|y\right]-\mathbb{E}\left[x|y\right] \mathbb{E}\left[x|y\right]^\top\right) = -\frac{1}{\sigma^2} I + \frac{1}{\sigma^4} \cov(X|y), \\
\nonumber   
    \nabla^2 (\log p) (y) &= -\mathbb{E}[\nabla^2 f(X)|y]+ \mathbb{E}[\nabla f(X) \nabla f(X)^\top|y]  + \mathbb{E}[\nabla f(X)|y]\mathbb{E}[\nabla f(X)|y]^\top \\
\label{app:eq:hessian-2}    
    &= -\mathbb{E}[\nabla^2 f(X)|y] + \cov(\nabla f(X)|y), \\
\nonumber   
    \nabla^2 (\log p) (y) &= -\frac{1}{\sigma^2} \left(\mathbb{E}\left[ (x-y) \nabla f(x)^\top|y\right]-\mathbb{E}\left[(x-y)|y\right] \mathbb{E}\left[\nabla f(x)|y\right]^\top\right)\\
\nonumber   
    &= -\frac{1}{\sigma^2} \left(\mathbb{E}\left[ x \nabla f(x)^\top|y\right]-\mathbb{E}\left[x|y\right] \mathbb{E}\left[\nabla f(x)|y\right]^\top\right) =-\frac{1}{\sigma^2} \cov(X,\nabla f(X)|y).
\end{align}

\subsection{Gaussian example}
If $X \sim \mathcal{N}(\mu, C)$, then $Y \sim \mathcal{N}(\mu, C+\sigma^2 I)$, and then
\begin{align*}
    \mathbb{E}[X|Y] &= \mu + C(C+\sigma^2 I)^{-1}(Y-\mu),\\
    \cov(X|Y) &= C-C(C+\sigma^2 I)^{-1} C.
\end{align*}
Therefore, $\mathbb{E}[X|Y]$ is Gaussian with mean $\mu$ and the covariance matrix $C(C+\sigma^2 I)^{-1}C$. We have $-\nabla^2 (\log p)(y) = (C+\sigma^2 I)^{-1}.$
    
\subsection{Distribution of $\mathbb{E}[X|Y]$} \label{app:p-vs-phat-m1}
If $Y$ is sampled from $p_Y$, let $\myhat{p}_\sigma$ be the distribution of $\mathbb{E}[X|Y]=Y+\sigma^2 \nabla (\log p)(Y)$. We can bound the 2-Wasserstein distance between $p$ and $\myhat{p}_\sigma$, by sampling $X$ from $p_X$, and taking $\myhat{x}=Y + \sigma^2 \nabla \log p(Y)$, where $Y\sim \mathcal{N}(X,\sigma^2 I)$. This leads to a particular ``coupling''~\citep{villani2021topics} between $p$ and $\myhat{p}_{\sigma}$. It follows:
\begin{equation} \nonumber
        W_2(p,\myhat{p}_{\sigma})^2 \myleq \mathbb{E}_{X,Y} \Vert \mathbb{E}[X|Y]-X \Vert^2.
\end{equation}
Given that the conditional expectation $ \mathbb{E}[X|Y]$ is the optimal estimator for the square loss, the last expression is less than the trival estimator $Y$, that is,  \<  \label{eq:p-vs-phat-m1} W_2(p,\myhat{p}_{\sigma})^2 \myleq \mathbb{E}_{X,Y} \Vert Y -X \Vert^2 = \sigma^2 d.\>

\subsection{\autoref{prop:p-vs-phat}} \label{app:p-vs-phat}
The calculations in \autoref{app:p-vs-phat-m1} that resulted in \eqref{eq:p-vs-phat-m1} leads to the following:
\begin{proof}[Proof for \autoref{prop:p-vs-phat}]
 Using \autoref{prop:universal}, we have
$$ W_2(p,\myhat{p}_{\sigma,m}) = W_2(p,\myhat{p}_{m^{-1/2}\sigma}),$$
which is then combined with \eqref{eq:p-vs-phat-m1}:
$$ W_2(p,\myhat{p}_{\sigma,m})^2 \myleq  \frac{\sigma^2}{m} d.$$
\end{proof}

\section{Proofs for \autoref{sec:geometry}} \label{app:gaussian-kappa}
\subsection{\autoref{prop:allatonce}}
\begin{proof}[Proof for \autoref{prop:allatonce}]
We start with an expression for $\log p(y_{1:m})$:
\< \label{eq:app:logp-aao} \begin{split}
     \log p(y_{1:m}) &= \log \int_{\mathbb{R}^d} \left( \prod_{k=1}^m \mathcal{N}(x;y_k,\sigma^2 I) \right)  \mathcal{N}(x;0,C) \;dx \\
     &= \sum_{i=1}^d \log  \int_\mathbb{R} \exp{\left(-\sum_{k=1}^m\frac{\left(y_{ki} - x_i\right)^2}{2\sigma^2}-\frac{x_i^2}{2\tau_i^2}\right)} dx_i  + \cst\\
     &= \sum_{i=1}^d \log \int_\mathbb{R} \exp{\left(-\frac{(x_i-\alpha_i)^2}{2 \beta_i^2}-\gamma_i\right)} dx_i + \cst.
\end{split}
\>
The expressions for $\alpha_i$, $\beta_i$, and $\gamma_i$ are given next by completing the square via matching second, first and zeroth derivative (in that order) of the left and right hand sides below
$$ -\sum_{t=1}^m\frac{\left(y_{ti} - x_i\right)^2}{2\sigma^2}-\frac{x_i^2}{2\tau_i^2} = -\frac{(x_i-\alpha_i)^2}{2 \beta_i^2}-\gamma_i $$
evaluated at $x_i=0$. The following three equations follow:
\begin{align*}
    \frac{1}{\beta_i^2} &= \frac{m}{\sigma^2} + \frac{1}{\tau_i^2},\\
    \frac{\alpha_i}{\beta_i^2} = \sum_{t=1}^m \frac{y_{ti}}{\sigma^2} \Rightarrow \alpha_i &=  \frac{1}{m+\sigma^2\tau_i^{-2}}\sum_{t=1}^m y_{ti},\\
    -\frac{\alpha_i^2}{2 \beta_i^2}-\gamma_i = -\sum_{t=1}^m \frac{y_{ti}^2}{2\sigma^2} \Rightarrow \gamma_i &= \sum_{t=1}^m \frac{y_{ti}^2}{2\sigma^2} -  \frac{1}{2\sigma^2(m+\sigma^2 \tau_i^{-2})}  \Bigl(\sum_{t=1}^m y_{ti}\Bigr)^2. 
\end{align*}
It is convenient to define:
\< \label{eq:tsigmatau} \tsigmatau_{ti} = \frac{1}{t+\sigma^2 \tau_i^{-2}}.\>
Using above expressions, \eqref{eq:app:logp-aao} simplifies to:
\< \label{eq:app:logpy1:m}
    \log p(y_{1:m}) = -\sum_{i=1}^d \gamma_i + \cst \\
    = -\sum_{t=1}^m \frac{\Vert y_t \Vert^2}{2\sigma^2}+ \frac{1}{2\sigma^2} \sum_{i=1}^d \tsigmatau_{mi} \Bigl(\sum_{t=1}^m y_{ti}\Bigr)^2 + \cst.
\>
The energy function can be written more compactly by introducing the matrix $F_{\sigma,m}$:
\< \nonumber
\begin{split} 
\log p(y_{1:m}) &= -\frac{1}{2} y_{1:m}^\top F_{\sigma,m} y_{1:m},\\ 
\sigma^{2}  [F_{\sigma,m}]_{ti,t'i'} &= \left( (1-\tsigmatau_{mi}) \delta_{tt'}-  \tsigmatau_{mi} (1-\delta_{tt'}) \right) \delta_{ii'}.
 \end{split} 
\>
In words, the $md \times md$ dimensional matrix $F_{\sigma,m}$ is block diagonal with $d$ blocks of size $m\times m$. The blocks themselves capture the interactions between different measurements indexed by $t$ and $t'$:
The $m\times m$ blocks of the matrix $F_{\sigma,m}$, indexed by $i\in[d]$ thus have the form:
$$ \sigma^2 F_{\sigma,m}^{(i)} = 
  (1-\tsigmatau_{mi}) I_m+ \tsigmatau_{mi} (I_m - 1_m1_m^\top ),
$$
where $1_m^\top=(1,1,\dots,1)$ is $m$-dimensional. It is straightforward to find the $m$ eigenvalues of the $m\times m$ matrix $F_{\sigma,m}^{(i)}$ for $i\in [d]$: \begin{itemize}
    \item $m-1$ degenerate eigenvalues equal to $\sigma^{-2}$ corresponding to the eigenvectors 
$$\{(1, -1, 0, \dots, 0)^\top, (1, 0, -1, \dots, 0)^\top, \dots, (1, 0, 0, \dots, -1)^\top\},$$
    \item one eigenvalue equal to $\sigma^{-2}(1-m \tsigmatau_{mi})$ corresponding to the eigenvector $(1, 1, \dots, 1)^\top$.
\end{itemize} 
Since $ m \tsigmatau_{mi}>0 $ we have:$$
    \lambda_{\rm max}(F_{\sigma,m}) = \sigma^{-2}, $$ which is $(m-1) d$ degenerate. The remaining $d$ eigenvalues are $\{\sigma^{-2}(1-m \tsigmatau_{mi})\}_{i\in[d]}$, 
    the smallest of which is given by
$$  \lambda_{\rm min}(F_{\sigma,m}) =  \sigma^{-2}\left(1- \frac{m}{m + \sigma^2 \tau_{\rm max}^{-2}}\right) =  \frac{\sigma^{-2}}{1+m\sigma^{-2} \tau_{\rm max}^2}.
$$
Thus we have:
$$ \kappa_{\sigma,m} =  \lambda_{\rm max}(F_{\sigma,m})/\lambda_{\rm min}(F_{\sigma,m})=1+m\sigma^{-2} \tau_{\rm max}^2. $$
\end{proof}

\subsection{\autoref{prop:oneatatime}}

\begin{proof}[Proof for \autoref{prop:oneatatime}]
Using \eqref{eq:app:logpy1:m} and \eqref{eq:tsigmatau} we have:
$$
-2\sigma^2 \log p(y_{1:t}) = \sum_{k=1}^t \Vert y_k \Vert^2 - \sum_{i=1}^d A_{ti} \Bigl(\sum_{k=1}^t y_{ki}\Bigr)^2.	
$$
Since $\log p(y_t|y_{1:t-1}) = \log p(y_{1:t})-\log p(y_{1:t-1})$, we have:
\< \nonumber 
\begin{split}
	-2\sigma^2 \log p(y_t|y_{1:t-1}) &= \Vert y_t\Vert^2-\sum_{i=1}^d A_{ti} y_{ti}\Bigl(y_{ti}+2 \sum_{k=1}^{t-1} y_{ki}\Bigr) + \cst \\
	&= \sum_{i=1}^d (1-A_{ti})\cdot\Bigl(y_{ti}-\frac{A_{ti}}{1-A_{ti}} \sum_{k=1}^{t-1} y_{ki} \Bigr)^2 + \cst,
\end{split} 
\>
Therefore the conditional density $p(y_t|y_{1:t-1})$ is the Gaussian $\mathcal{N}(\mu_{t|t-1},F_{t|t-1}^{-1})$ with a shifted mean
$$ \mu_{t|t-1} = \Bigl(\frac{A_{ti}}{1-A_{ti}}\sum_{k=1}^{t-1} y_{ki} \Bigr)_{i\in[d]},$$ and with an anisotropic, diagonal covariance/precision matrix whose spectrum is given by:
$$ \sigma^{2} \lambda_i(F_{t|t-1}) =  1-(t+\sigma^2 \tau_i^{-2})^{-1}.$$
Thus:
$$
	\kappa_{t|t-1}= \frac{\lambda_{\max}(F_{t|t-1})}{\lambda_{\min}(F_{t|t-1})}=\frac{1-(t+\sigma^2 \tau_{\min}^{-2})^{-1}}{1-(t+\sigma^2 \tau_{\max}^{-2})^{-1}}.
$$
Lastly, to prove monotonicity result \eqref{eq:anisotropic-monotonicity} we do an analytic continuation of $\kappa_{t|t-1}$ to continuous values by defining $\eta(t) = \kappa_{t|t-1}$ and taking its derivative, below $R=\sigma^2 \tau_{\min}^{-2}, r=\sigma^2 \tau_{\max}^{-2}$, thus $R>r$:
\< \nonumber 
\begin{split}
\eta'(t)&=\frac{(t+R)^{-2}}{1-(t+r)^{-1}}-\frac{(1-(t+R)^{-1})(t+r)^{-2}}{(1-(t+r)^{-1})^2}\\
&= \frac{(t+r)(t+r-1)}{(t+R)^2(t+r-1)^2}-\frac{(t+R-1)(t+R)}{(t+r-1)^2(t+R)^2}\\
&= \frac{(t+r)^2-r-(t+R)^2+R}{(t+r-1)^2(t+R)^2}	\\
&= \frac{(r-R)(2t+r+R-1)}{(t+r-1)^2(t+R)^2}<0.
\end{split}
\>
\end{proof}

\section{Proofs for \autoref{sec:log-concave}} \label{app:logconcave}
\subsection{\autoref{lemma:log-concave}}
\begin{proof}[Proof for \autoref{lemma:log-concave}]
We would like to find an upper bound for $\nabla^2 (\log p) (y)$. By using $$\nabla^2 (\log p) (y)=-\sigma^{-2} I+ \sigma^{-4}\cov(X|y),$$ derived in \autoref{app:score-hessian}, we need to upper bound $\cov(X|y)$. It suffices to study $\mathbb{E}[\Vert X-x_0\Vert^2|y]$ since $$\cov(X|y)=\cov(X-x_0|y) \mypreceq \mathbb{E}[\Vert X-x_0\Vert^2|y]\, I.$$ 

We have by a convexity argument\footnote{We consider the exponential family $p(x|\nu)=\exp\bigl(-f(x)+\frac{1}{\sigma^2}(x-x_0)^\top (y-x_0)-\frac{\nu}{2\sigma^2}\Vert x-x_0\Vert^2-a(\nu)\bigr).$ Since $a(\nu)$ is convex~\citep{wainwright2008graphical}, we have $a'(1) \mygeq a'(0)$, which is exactly the desired statement.}:
\< \label{eq:exp-family}
\begin{split}
\mathbb{E}[\Vert X-x_0\Vert^2|y] &= \frac{\int e^{-f(x)-\frac{1}{2\sigma^2}\Vert x-y\Vert^2} \Vert x-x_0\Vert^2 dx}{\int e^{-f(x)-\frac{1}{2\sigma^2}\Vert x-y\Vert^2} dx}\\
&\myleq \frac{\int e^{-f(x)+\frac{1}{\sigma^2}(x-x_0)^\top (y-x_0)} \Vert x-x_0\Vert^2 dx}{\int e^{-f(x)+\frac{1}{\sigma^2}(x-x_0)^\top (y-x_0)} dx}.
\end{split}
\>


Next, we find an upper bound for $\Vert x - x_0 \Vert^2$ itself. Using the assumption in the lemma we have:
$$ \Vert \nabla f(x)+ \sigma^{-2}(x_0-y) \Vert \mygeq \Vert \nabla f(x) \Vert - \sigma^{-2} \Vert x_0-y \Vert \mygeq \mu \Vert x - x_0 \Vert - \Delta - \sigma^{-2}  \Vert x_0-y \Vert, $$
leading to $ \mu \Vert x - x_0 \Vert \myleq \Vert \nabla f(x)+\sigma^{-2}(x_0-y) \Vert + \Delta + \sigma^{-2} \Vert x_0-y \Vert $ and thus 
$$ \mu^2 \Vert x - x_0 \Vert^2 \myleq 3 \Vert \nabla f(x)+ \frac{x_0-y}{\sigma^2} \Vert^2+ 3 \Delta^2 + 3 \frac{\Vert x_0-y\Vert^2}{\sigma^4}. $$
Finally, using \eqref{eq:exp-family}, we only need to find an upper-bound for $\Vert \nabla f(x)+ \sigma^{-2}(x_0-y) \Vert^2$ under the distribution $\tilde{p}(x) \propto e^{-\tilde{f}(x)}$, where $\tilde{f}(x)= f(x)-\sigma^{-2}(x-x_0)^\top(y-x_0)$. This is achieved with:
$$ \int \tilde{p}(x) \Vert \nabla f(x)+\frac{x_0-y}{\sigma^2} \Vert^2 dx=\int \tilde{p}(x) \Vert \nabla \tilde{f}(x) \Vert^2 dx = \int \tilde{p}(x)\, \tr \nabla^2 f(x) dx \myleq  L d, $$
where second equality is obtained using integration by parts akin to score matching~\citep{hyvarinen2005estimation}, and for the last inequality we used our assumption $\nabla^2 f(x) \mypreceq L I$.
Putting all together we arrive at:
$$\nabla^2 (\log p)(y) \mypreceq \Bigl(-1 + \frac{3Ld}{\mu^2 \sigma^2}+\frac{3 \Delta^2}{\mu^2 \sigma^2}+\frac{3\Vert x_0-y\Vert^2}{\mu^2 \sigma^6}\Bigr) \frac{I}{\sigma^2}.$$
\end{proof}

\subsection{\autoref{lemma:more-logconcave}}
Since $\log p(y_t|y_{1:t-1}) = \log p(y_{1:t}) - \log p(y_{1:t-1})$, we have $$ \nabla_{y_t}^k \log p(y_t|y_{1:t-1}) = \nabla_{y_t}^k \log p(y_{1:t}).$$ 
Start with the score function ($k=1$):
$$
     \nabla_{y_t} \log p(y_{1:t}) = \frac{\int  \sigma^{-2} (x-y_t) p(x) \prod_{i=1}^t p(y_i|x)  dx}{\int  p(x) \prod_{i=1}^t p(y_i|x)  dx} = \sigma^{-2} \left(\mathbb{E}[X|y_{1:t}] - y_t \right).
$$
Next, we derive the Hessian $\nabla_{y_t}^2 \log p(y_{1:t})=- \sigma^{-2} I + A_t - B_t $, where $A_t$ and $B_t$ are given by:
\begin{align*} 
     A_t &= \frac{\int  \sigma^{-4} (x-y_t)(x-y_t)^\top p(x) \prod_{i=1}^t p(y_i|x)  dx}{\int  p(x) \prod_{i=1}^t p(y_i|x)  dx} = \sigma^{-4}\, \mathbb{E}[(X-y_t)(X-y_t)^\top|y_{1:t}], \\
    B_t &= \sigma^{-4} \left(\frac{\int   (x-y_t) p(x) \prod_{i=1}^t p(y_i|x)  dx}{\int  p(x) \prod_{i=1}^t p(y_i|x)  dx}\right) \left(\frac{\int   (x-y_t) p(x) \prod_{i=1}^t p(y_i|x)  dx}{\int  p(x) \prod_{i=1}^t p(y_i|x)  dx}\right)^\top \\ \nonumber
    &= \sigma^{-4} \left(\mathbb{E}[X|y_{1:t}] - y_t \right)  \left(\mathbb{E}[X|y_{1:t}] - y_t \right) ^\top.
\end{align*}
By simplifying $A_t-B_t$, the $y_t$ cross-terms cancel out and the posterior covariance matrix emerges:
\<  \label{eq:hessian-t}
    \nabla_{y_t}^2 \log p(y_{1:t}) =  - \sigma^{-2} I + \sigma^{-4} \cov(X|y_{1:t}).
\>
The lemma is proven since the mean of the posterior covariance $\mathbb{E}_{y_{1:t}} \cov(X|y_{1:t})$ can only go down upon accumulation of measurements (conditioning on more variables).

\subsection{\autoref{theorem:compact}} \label{app:lemma-compact}
\begin{proof}[Proof for \autoref{theorem:compact}]
We start with start with the definition of $\sigmam$-density:
$$
p(y_{1:m}) = \int_\mathcal{Z} p(z) p(y_{1:m}|z) dz \propto \int_\mathcal{Z} p(z) \left(\int_\mathcal{X} \exp\biggl(-\frac{1}{2\tau^2} \Vert x-z\Vert^2-\frac{1}{2\sigma^2} \sum_{t=1}^m \Vert x - y_t \Vert^2 \biggr) dx \right) dz,
$$
which we express by integrating out $x$. We have 
\< \nonumber 
\begin{split}
p(y_{1:m}|z) &\propto  \int_\mathcal{X} \exp\biggl(-\frac{1}{2\tau^2} \Vert x-z\Vert^2-\frac{m}{2\sigma^2} \Bigl(\Vert x - \overline{y}_{1:m}\Vert^2  -\Vert \overline{y}_{1:m} \Vert^2 + \frac{1}{m} \sum_{t=1}^m \Vert y_t\Vert^2\Bigr) \biggr) dx \\
&\propto \exp\left( -\frac{m}{2(m\tau^2+\sigma^2)}\Vert z - \overline{y}_{1:m}\Vert^2+\frac{m}{2\sigma^2} \Vert \overline{y}_{1:m} \Vert^2 - \frac{1}{2\sigma^2}  \sum_{t=1}^m \Vert y_t\Vert^2 \right).
\end{split}
\>
We can then express $\nabla_{y_m} \log p(y_{1:m})$ in terms of $\mathbb{E}[Z|y_{1:m}]$:
$$\nabla_{y_m} \log p(y_{1:m})=\frac{m\tau^2}{\sigma^2(m\tau^2+\sigma^2)}\overline{y}_{1:m}-\frac{y_m}{\sigma^2}+\frac{1}{m\tau^2+\sigma^2} \mathbb{E}[Z|y_{1:m}].$$
Next, we take another derivative:
$$  \nabla^2_{y_m} \log p (y_{1:m}) = \Bigl(\frac{\tau^2}{\sigma^2(m\tau^2+\sigma^2)}-\frac{1}{\sigma^2}\Bigr) \cdot I+\frac{1}{m\tau^2+\sigma^2} \nabla_{y_m} \mathbb{E}[Z|y_{1:m}].$$
Next, we compute $\nabla_{y_m} \mathbb{E}[Z|y_{1:m}]$:
\< \nonumber
\begin{split}
	\nabla_{y_m} \mathbb{E}[Z|y_{1:m}] &= \frac{1}{m\tau^2+\sigma^2}\mathbb{E}[ZZ^\top|y_{1:m}]+\frac{m\tau^2}{m\tau^2+\sigma^2} \mathbb{E}[Z|y_{1:m}] \overline{y}_{1:m}^\top-\frac{1}{\sigma^2} \mathbb{E}[Z|y_{1:m}] y_m^\top \\
	&-\mathbb{E}[Z|y_{1:m}] \left( \frac{1}{m\tau^2+\sigma^2} \mathbb{E}[Z|y_{1:m}]+\frac{m\tau^2}{m\tau^2+\sigma^2} \overline{y}_{1:m}-\frac{1}{\sigma^2} y_m\right)^\top \\
	&= \frac{1}{m\tau^2+\sigma^2}  \cov(Z|y_{1:m}).
\end{split}
\>
Putting all together, we arrive at:\footnote{Not required for the proof, but we can also  relate $\cov(X|y_{1:m})$ and $\cov(Z|y_{1:m})$ directly since as we know (see the proof of \autoref{lemma:more-logconcave}):
\< \label{eq:covz:2} \nabla^2_{y_m} \log p (y_{1:m}) = -\frac{1}{\sigma^2} I + \frac{1}{\sigma^4} \cov(X|y_{1:m}).\>
By combining \eqref{eq:covz:1} and \eqref{eq:covz:2}, it follows:
$$ \cov(X|y_{1:m}) = \frac{\sigma^2\tau^2}{m\tau^2+\sigma^2} I + \Bigl(\frac{\sigma^2}{m\tau^2+\sigma^2}\Bigr)^2 \cov(Z|y_{1:m}).$$}
\< \label{eq:covz:1} \nabla^2_{y_m} \log p (y_{1:m}) = \frac{1}{\sigma^2}\Bigl(\frac{\tau^2}{m\tau^2+\sigma^2}-1\Bigr) \cdot I + \frac{1}{(m\tau^2+\sigma^2)^2} \cov(Z|y_{1:m}).  \>
Finally, since $\Vert Z \Vert^2 \myleq R^2$ almost surely, we have $\cov(Z|y_{1:m}) \mypreceq R^2 I,$ therefore $$ \nabla_{y_m}^2 \log p(y_m|y_{1:m-1}) = \nabla_{y_m}^2 \log p(y_{1:m}) \mypreceq \zeta(m) I,$$
where
$$ \zeta(m) = \frac{1}{\sigma^2}\Bigl(\frac{\tau^2}{m \tau^2+\sigma^2}-1\Bigr)+\frac{R^2}{(m \tau^2+\sigma^2)^2}. $$
\end{proof}
%

\subsection{Derivation for \eqref{eq:mog-hess}}
In this example, we have $X=Z+N_0$, $N_0 \sim \mathcal{N}(0,\tau^2 I)$, and $Y=X+N_1$, $N_1 \sim \mathcal{N}(0,\sigma^2 I)$, therefore
\< \label{app:eq:y=z+n} Y = Z + N,\; N\sim \mathcal{N}(0,(\sigma^2 + \tau^2) I),\>
where $Z \sim p(z),$
$$ p(z) = \frac{1}{2} \delta(z-\mu) + \frac{1}{2} \delta(z+\mu),$$
$\delta$ is the Dirac delta function in $d$-dimensions. Alternatively, we have $p(y) = \int p(z) p(y|z) dz$, where
$$ p(y|z) = \mathcal{N}(y; z, (\sigma^2 + \tau^2) I).$$
Using \eqref{app:eq:hessian-1}, adapted for \eqref{app:eq:y=z+n}, we have:
\< \label{eq:app:mog-hess} \hess(y) = -\nabla^2 (\log p) (y) = \frac{1}{\sigma^2 +\tau^2}\Bigl( I - \frac{1}{\sigma^2+\tau^2} \cov(Z|y) \Bigr).\>
Next, we derive an expression for $\cov(Z|y)$:
$$
	\cov(Z|y) = \mathbb{E}[Z Z^\top|y] - \mathbb{E}[Z|y] \mathbb{E}[Z|y]^\top.
$$
We have:
\begin{align*}
	\mathbb{E}[Z Z^\top|y] &= \mu \mu^\top, \\
	\mathbb{E}[Z|y] &= \mu \cdot \frac{e^{-A}-e^{-B}}{e^{-A}+e^{-B}} ,\\
	A &= \frac{\Vert y-\mu\Vert^2}{2(\sigma^2+\tau^2)},\\
	B &= \frac{\Vert y+\mu\Vert^2}{2(\sigma^2+\tau^2)}.
\end{align*}
It follows:
\<
\label{eq:app:mog-cov}
\begin{split}
	\cov(Z|y) &= \mu \mu^\top \cdot \biggl( 1- \Bigl( \frac{e^{-A}-e^{-B}}{e^{-A}+e^{-B}} \Bigr)^2 \biggr) 
	= \frac{2\mu \mu^\top}{1+\cosh(B-A)} \\
	&= 2\mu \mu^\top \cdot \biggl(1+\cosh\Bigl(\frac{2 \mu^\top y}{\sigma^2+\tau^2}\Bigr)\biggr)^{-1}.
\end{split}
\>
By combining \eqref{eq:app:mog-hess} and \eqref{eq:app:mog-cov} we arrive at \eqref{eq:mog-hess}. 


\section{Detailed algorithm} \label{app:detailed_algorithm}

\begin{algorithm}[H]
\caption{\mcmc$_\sigma$ in Algorithm~\ref{alg:oat} via Underdamped Langevin MCMC by Sachs et al. (2017)}
  \begin{algorithmic}[1]
  \STATE \textbf{Input} $Y_t^{(i-1)}$, ${\overline Y}_{1:t-1}$ 
    \STATE \textbf{Parameters} current MCMC iteration $i$, current measurement index $t$, step size $\delta$, friction $\gamma$, steps taken $n_t$, smoothed score function ${\myhat{g}}_1(y; \sigma)$ \eqref{eq:ghat}, Lipschitz parameter $L$, noise level $\sigma$
    \STATE \textbf{Output} $Y_t^{(i)}$
    \STATE Initialize $Y^{(i, 0)}_t \sim \text{Unif}([0,1]^{d}) + \mathcal{N}(0,\sigma^2 I) $
    \STATE Initialize $V \leftarrow 0$
    \FOR{$k=[0,\dots, K-1]$}
      \STATE $Y^{(i, k + 1)}_t  = Y^{(i, k)}_t + \frac{\delta}{2}\,V$
      \STATE ${{\overline Y}_{1:t}} \leftarrow {{\overline Y}_{1:t-1} + (Y_t^{(i, k+1)} - {\overline Y}_{1:t-1})/t}$
      \STATE $ G   \leftarrow  t^{-1}{\hat g}_1({\overline Y}_{1:t}; t^{-1/2}\sigma) + \sigma^{-2} ({\overline Y}_{1:t} - Y_t^{(i, k+1)})$ according to \eqref{eq:ghat_oat}
      \vspace{1mm}
      \STATE $V \leftarrow V  + \frac{\delta}{2L}\,{G} $
      \vspace{2mm}
      \STATE $ B \sim \mathcal{N}(0,I) $ \\
      \STATE $ V \leftarrow \exp(- \gamma \delta)\, V + \frac{\delta}{2L}\,{ G} + \sqrt{ \frac{1}{L}\left(1-\exp(-2 \gamma \delta )\right)} { B}$
      \STATE $ Y^{(i, k + 1)}_t \leftarrow Y^{(i, k + 1)}_t + \frac{\delta}{2}\, V$
    \ENDFOR
    \RETURN{$Y_t^{(i)} = Y_t^{(i, K)}$}
  \end{algorithmic}
 \label{alg:detailed_alg}
\end{algorithm}

\section{Experimental detail} \label{app:experimental_detail}

We fixed the Lipschitz parameter at $L = \frac{1}{\sigma^2}$. The other parameters were tuned on a log-spaced grid. We searched the step size $\delta$ over $\{ 1\text{e-}4, 3\text{e-}4, 1\text{e-}3, 3\text{e-}3, 1\text{e-}2 \}$ and the effective friction $\gamma \delta$ over $\{0.0125, 0.025, 0.05, 0.1, 0.2, 0.4, 0.8, 1.6\}$. We found that $\delta$ of $0.03$ and $\gamma \delta$ of 0.05 ($\gamma = 5/3$) worked well for most configurations of test density type, $\sigma$, $d$, and MCMC algorithm. For AAO, $m$ ran over $\{200, 400, 600, 800, 1000\}$. For OAT, we fixed $n_t = 1000$ for all $t$.

The hyperparameters were tuned for each algorithm, but the total number of iterations was kept fixed. We define each iteration as an MCMC update step. For OAT, the total number of iterations is $\sum_{t=1}^m n_t$, the number of MCMC iterations for each measurement $t$, $n_t$, summed up over the $m$ measurements. We had $n_t=100$ for all $t$ and $m=1000$. For the remaining three sampling schemes (AAO, $m=1$, and $\sigma=0$), the total number of iterations is simply the number of MCMC iterations. 

We found a cold initialization of $p(y_t|y_{1:t-1})$ samples for each Markov phase $m$ to work better than a warm initialization at the final samples from the previous phase $t-1$. We thus re-initialized $y_t| y_{1:t-1} \sim  {\rm Unif}([-1, 1]^d) + \mathcal{N}(0, \sigma^2 I)$ for every $t$.

\section{MCMC algorithms} \label{app:mcmc_alg}

Our algorithm is agnostic to the choice of MCMC sampling algorithm used in the Markovian phases. In this section, we run OAT sampling with four different Langevin MCMC algorithms. The results presented earlier in \autoref{sec:experiments} uses an ULD algorithm with an Euler discretization scheme that extends the BAOAB integration using multiple time steps for the O-part (``Sachs et al.'') \citep{sachs2017langevin}. Next, we consider two algorithms that operate on the integral representations of ULD. Recall that continuous-time ULD is represented by the following stochastic differential equation (SDE):
\begin{align*}
    d v_t &= -\gamma v_t dt - u \nabla f(x_t) dt + (\sqrt{2 \gamma u}) d B_t, \nonumber \\
    d x_t &= v_t dt,
\end{align*}
where $x_t, v_t \in \mathbb{R}^d$ and $B_t$ is the standard Brownian motion in $\mathbb{R}^d$. The solution $(x_t, v_t)$ to the continuous-time ULD is
\begin{align} \label{eq:continuous_integral_uld}
    v_t &= v_0 e^{-\gamma t} - u\left(\int_0^t \exp \left(-(t-s)\right) \nabla f(x_s) ds \right) + \sqrt{2 \gamma u} \int_0^t \exp\left(-\gamma (t-s)\right) d B_s, \nonumber \\
    x_t &= x_0 + \int_0^t v_s ds.
\end{align}

Similarly, the discrete ULD is defined by the SDE
\begin{align*}
    d {\tilde v_t} &= -\gamma v_t dt - u \nabla f({\tilde x_0}) dt + (\sqrt{2 \gamma u}) d B_t, \nonumber \\
    d {\tilde x_t} &= {\tilde v_t} dt,
\end{align*}
which yields the solution
\begin{align} \label{eq:discrete_integral_uld}
    {\tilde v_t} &= {\tilde v_0} e^{-\gamma t} - u\left(\int_0^t \exp \left(-(t-s)\right) \nabla f({{\tilde x}_0}) ds \right) + \sqrt{2 \gamma u} \int_0^t \exp\left(-\gamma (t-s)\right) d B_s, \nonumber \\
    {\tilde x_t} &= {{\tilde x}_0} + \int_0^t {\tilde v_s} ds.
\end{align}
\cite{shen2019randomized} seeks a lower discretization error by using a 2-step fixed point iteration method, or the randomized midpoint method. The integral in \eqref{eq:continuous_integral_uld} is evaluated along uniform random points between 0 and $t$. On the other hand, \citet{cheng2018underdamped}  computes the moments of the joint Gaussian over $({\tilde x}_t, {\tilde v}_t)$ in the updates of \eqref{eq:discrete_integral_uld}. In our comparison we additionally include MALA, an Euler discretization of the overdamped Langevin dynamics represented by the SDE
\begin{align} \label{eq:sde_uld_continuous}
    d x_t = - u \nabla f(x_t) dt + (\sqrt{2 \gamma u}) d B_t,
\end{align}
accompanied by Metropolis adjustment to correct for the discretization errors \citep{roberts1996exponential}.


\begin{figure}[t!]
\begin{center}
{\includegraphics[width=0.49\textwidth]{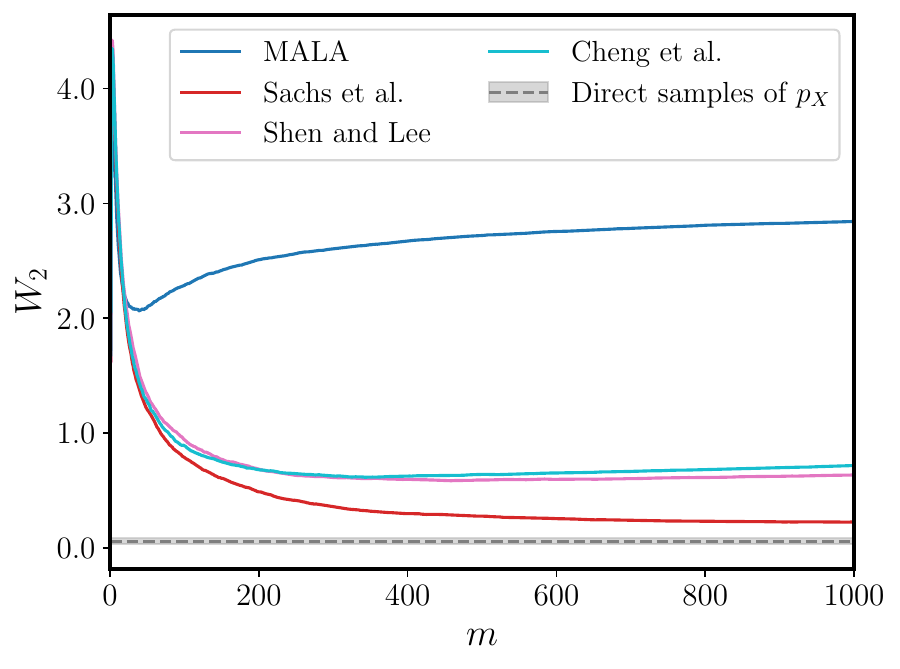}}
\end{center}
\caption{\label{fig:mcmc_alg} $W_2$ vs. $m$ for various MCMC$_\sigma$ algorithms used in the inner loop of Algorithm~\ref{alg:oat}.}
\end{figure}

\begin{figure}[ht!]
\begin{center}
\begin{subfigure}[MALA~\citep{roberts1996exponential}]
{\includegraphics[width=0.49\textwidth]{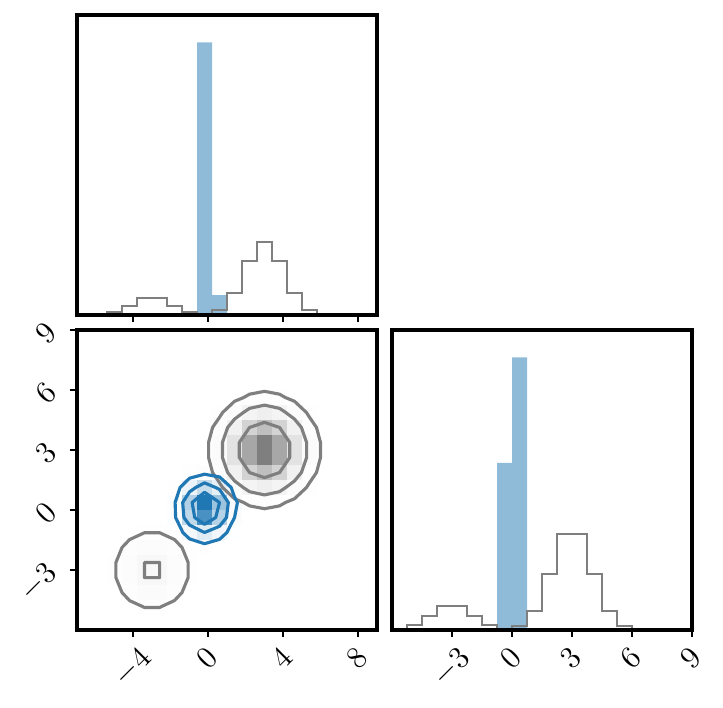}}
\end{subfigure}
\begin{subfigure}[\citet{cheng2018underdamped}]
{\includegraphics[width=0.49\textwidth]{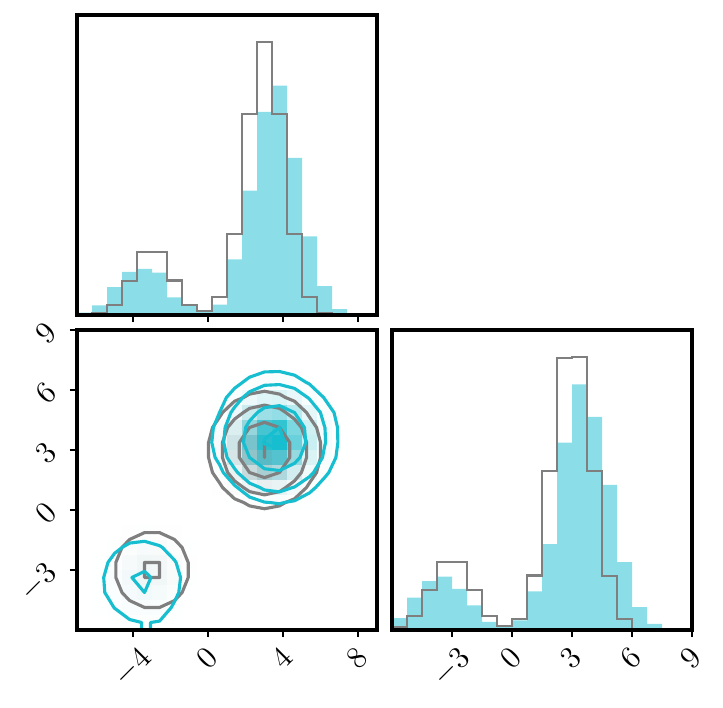}}  
\end{subfigure}
\begin{subfigure}[\citet{shen2019randomized}]
{\includegraphics[width=0.49\textwidth]{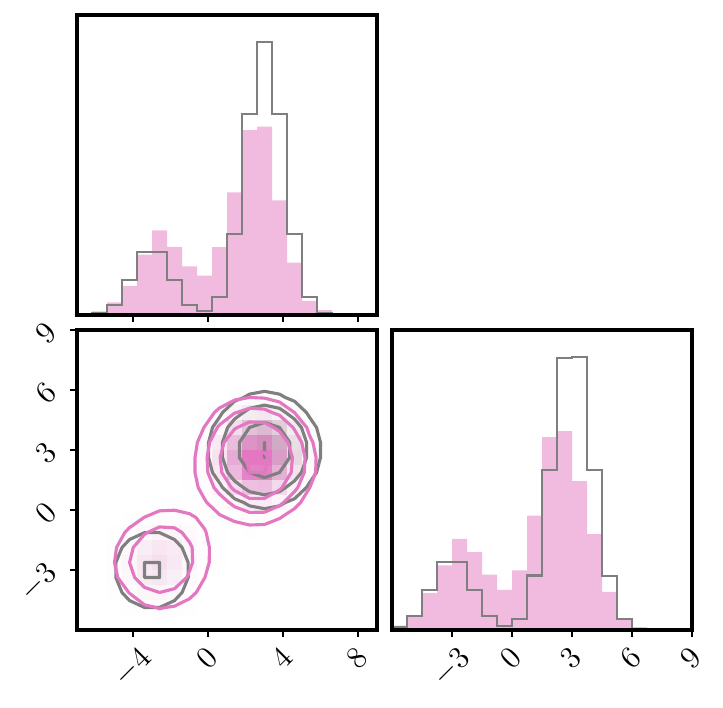}}
\end{subfigure}
\begin{subfigure}[\citet{sachs2017langevin}]
{\includegraphics[width=0.49\textwidth]{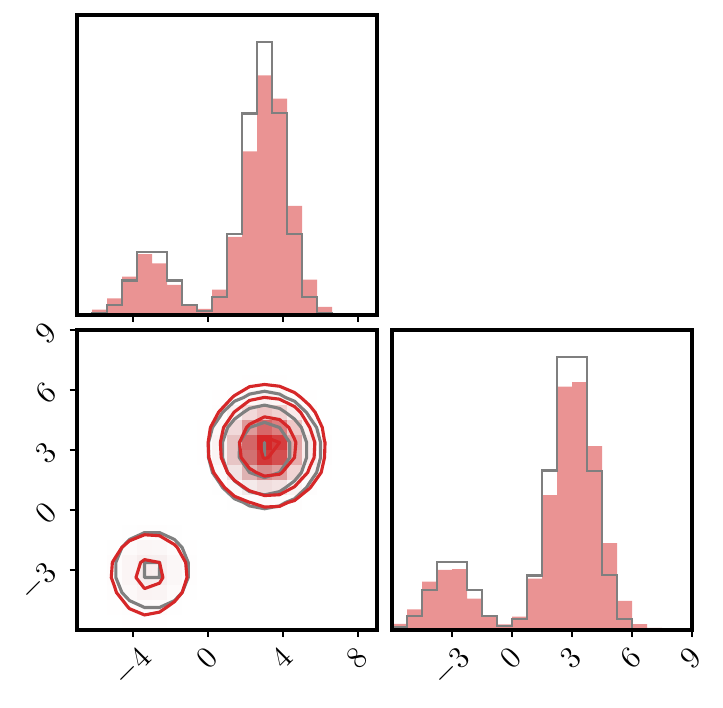}}
\end{subfigure}
\end{center}
\caption{\label{fig:mc_samples} Final $\myhat X$ samples for various MCMC$_\sigma$ algorithms used in the inner loop of Algorithm~\ref{alg:oat}.}
\end{figure}

The algorithms are compared in \autoref{fig:mcmc_alg} for the Gaussian mixture test density introduced in \autoref{sec:gmm} with $d=8$. The three (unadjusted) ULD algorithms converge faster than does MALA to a lower $W_2$. In ULD algorithms, Brownian motion affects the positions $x_t$ through the velocities $v_t$, rather than directly as in MALA, resulting in a smoother evolution of $x_t$ that lends itself better to discretization. The first two dimensions of the final samples are displayed in \autoref{fig:mc_samples}. MALA samples fail to separate into the two modes, whereas Cheng et al, Shen and Lee, and Sachs et al have better sample quality, with Sachs et al performing the best and almost approaching the sample variance of $W_2$ when samples are directly drawn from $p_X$ (the gray band). 

\begin{figure}[t!]
\begin{center}
\begin{subfigure}[$W_2$ vs. $m$]
{\includegraphics[width=0.49\textwidth]{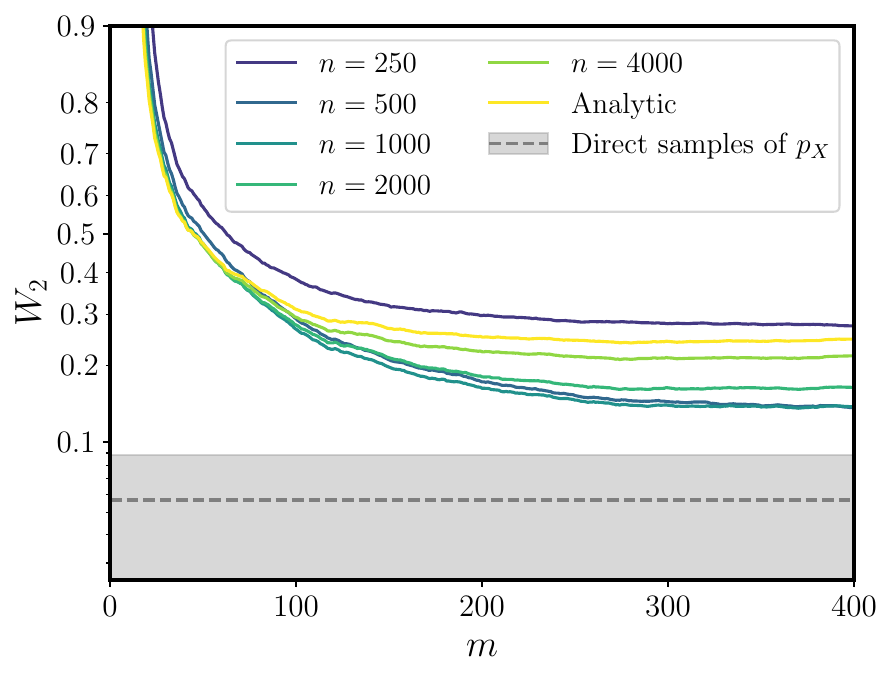}}
\end{subfigure}
\begin{subfigure}[Final $W_2$ vs. number of MC samples]
{\includegraphics[width=0.49\textwidth]{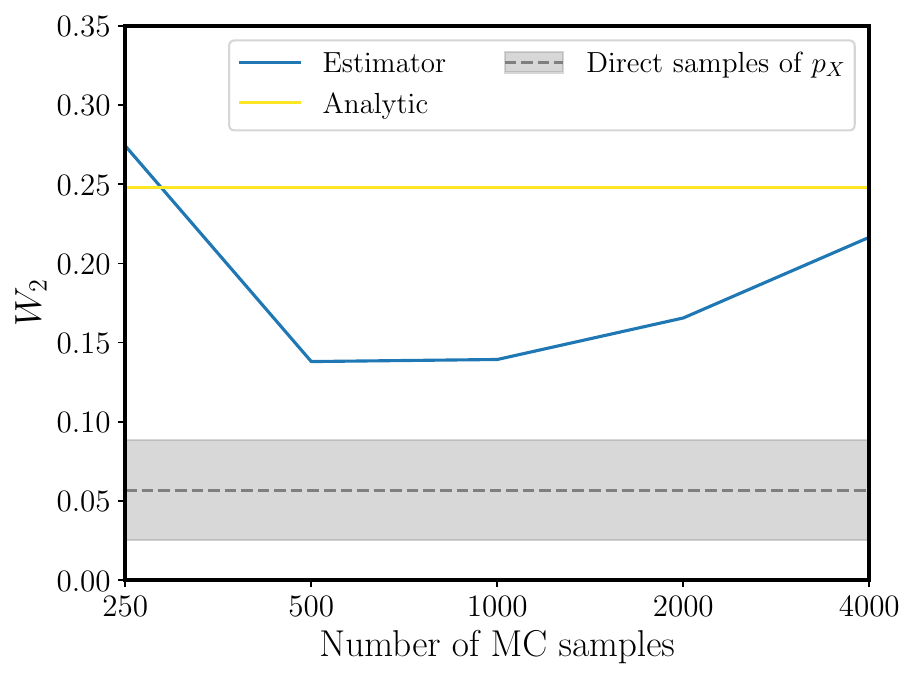}}  
\end{subfigure}
\end{center}
\caption{\label{fig:mc_samples_dim2} Effect of increasing the number of MC samples in the estimator of $\nabla \log p(y)$ for the mixture of Gaussians test density introduced in \autoref{sec:gmm} with $d=2$}
\end{figure}

\begin{figure}[t!]
\begin{center}
\begin{subfigure}[$W_2$ vs. $m$]
{\includegraphics[width=0.49\textwidth]{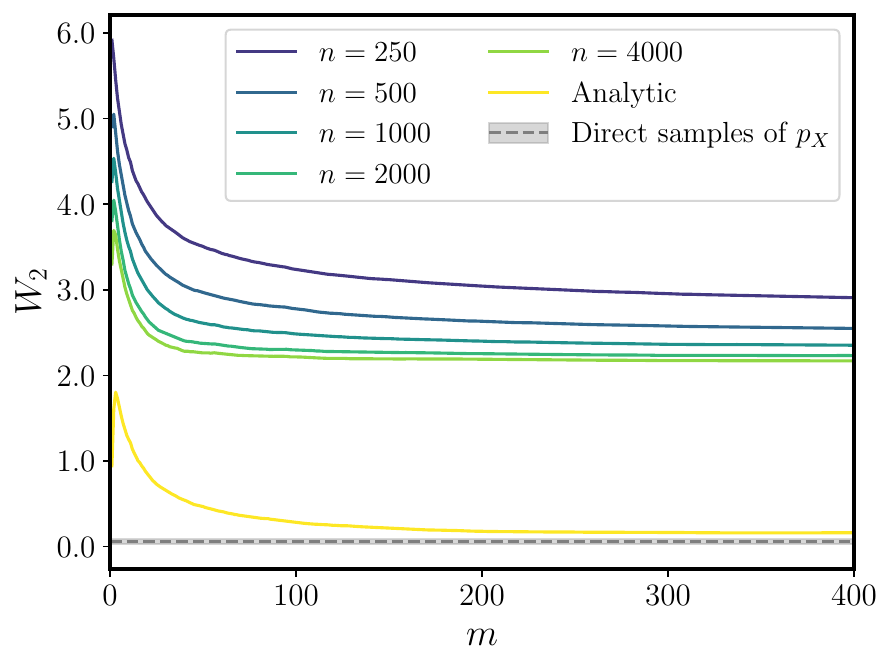}}
\end{subfigure}
\begin{subfigure}[Final $W_2$ vs. number of MC samples]
{\includegraphics[width=0.49\textwidth]{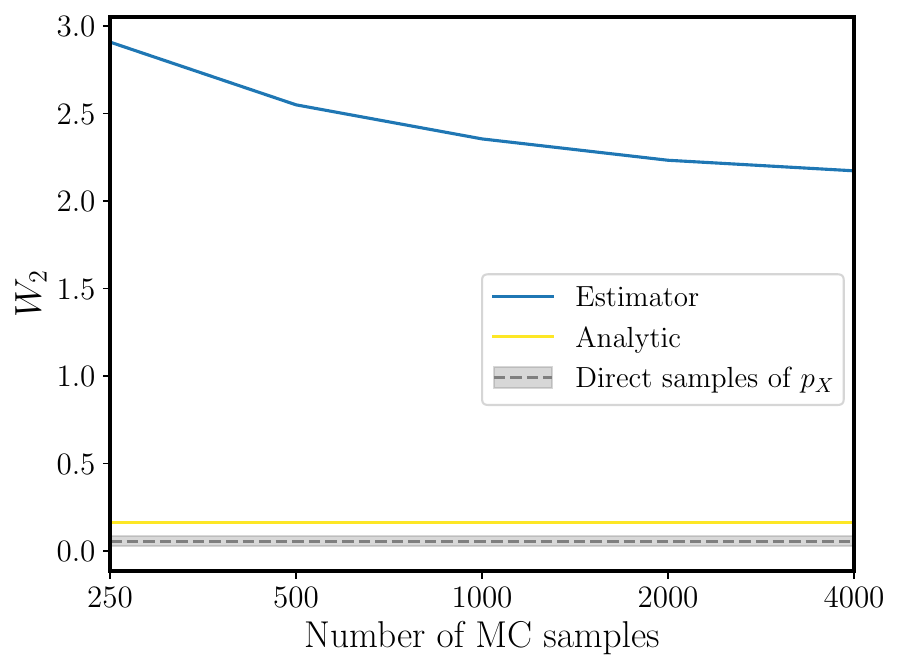}}  
\end{subfigure}
\end{center}
\caption{\label{fig:mc_samples_dim8} Effect of increasing the number of MC samples in the estimator of $\nabla \log p(y)$ for the mixture of Gaussians test density introduced in \autoref{sec:gmm} with $d=8$}
\end{figure}

\section{Score estimation} \label{app:score}

Langevin MCMC requires the smoothed score function $g(y;\sigma) = \nabla \log p(y)$. Results presented earlier in \autoref{sec:experiments} assumed access to the analytic score function. Here, we use the estimator presented in \eqref{eq:ghat_oat} with varying numbers of MC samples $n$. 

To prevent numerical underflow, we implemented \eqref{eq:ghat} as follows:
\begin{align}
A &= \logsumexp_{i=1}^{n} \bigl( -f(y+\sigma \epsilon_i)  \bigr), \\
B_+ &= \logsumexp_{j=1}^{n_+} \bigl(\log \epsilon_j - f(y+\sigma \epsilon_j)  \bigr), \label{eq:pos_term} \\
B_- &= \logsumexp_{k=1}^{n_-} \bigl(\log -\epsilon_k - f(y+\sigma \epsilon_k)  \bigr), \label{eq:neg_term} \\
{\myhat{g}}_1(y; \sigma) &= \frac{1}{\sigma} \cdot \Bigl( e^{{B_+} - A} - e^ {{B_-} - A}  \Bigr),
\end{align}
where $\logsumexp_i(a_i) \coloneqq a_{\rm max} + \log \sum_i \exp(a_i - a_{\rm max})$ with $a_{\rm max} \coloneqq {\rm max}_i \ a_i$. In \eqref{eq:pos_term} and \eqref{eq:neg_term}, $j=1, \dotsc, n_+$ and $k=1, \dotsc, n_-$ denote the indices for which $\epsilon$ is positive and negative, respectively, with $n_+ + {n_-} = n$. Note that the same Gaussian samples $\epsilon$ were used to evaluate the numerator and the denominator.

For a Gaussian mixture density introduced in \autoref{sec:gmm} with $d=2$, both the analytic and the estimated score functions converge to a low $W_2$, as shown in \autoref{fig:mc_samples_dim2}. The sample quality is on par with the analytic score function with $n$ as small as 500 and there is little benefit to increasing the $n$ past 500. 

For $d=8$, the estimated score reveals signs of error, as  \autoref{fig:mc_samples_dim8} shows. There is decreasing marginal return in $W_2$ as $n$ doubles, although higher $n$ does help reduce $W_2$ to some extent. 

Variance reduction techniques, such as importance reweighting, may help get more mileage from finite $n$. 



\section{Mixture of correlated Gaussians} \label{app:gmm_correlated}

In this section, we study a correlated test density, namely a mixture of two 2-dimensional Gaussians with full covariances:
\< \label{eq:gmm_correlated}
p(x) = \alpha\,\mathcal{N}(x;\mu,\Sigma_0)+ (1-\alpha)\, \mathcal{N}(x;-\mu,\Sigma_1).
\>

We choose $\mu=3\cdot 1_d$, $\Sigma_0 = R \ {\rm diag}(\frac{1}{4}, 4) \ R^T$, $\Sigma_1 = R^T\  {\rm diag}(1, 9) \  R$, $\alpha=\frac{4}{5}$, and the rotation matrix $$R = \begin{pmatrix} \cos \theta & -\sin \theta \\
\sin \theta & \cos \theta
\end{pmatrix}$$ with $\theta=\pi/360$. \autoref{fig:gmm_correlated} compares the performance of the same Langevin MCMC algorithm \citep{sachs2017langevin} used within the OAT scheme (a) with that used without Gaussian smoothing (b). In both cases here, we initialized the samplers at the origin (5000 particles total). The total number of iterations are the same between two algorithms (see~\autoref{app:experimental_detail} for details).
\begin{figure}[h!]
\begin{center}
\begin{subfigure}[OAT, $\sigma=16$, $m=1000$]
{\includegraphics[width=0.49\textwidth]{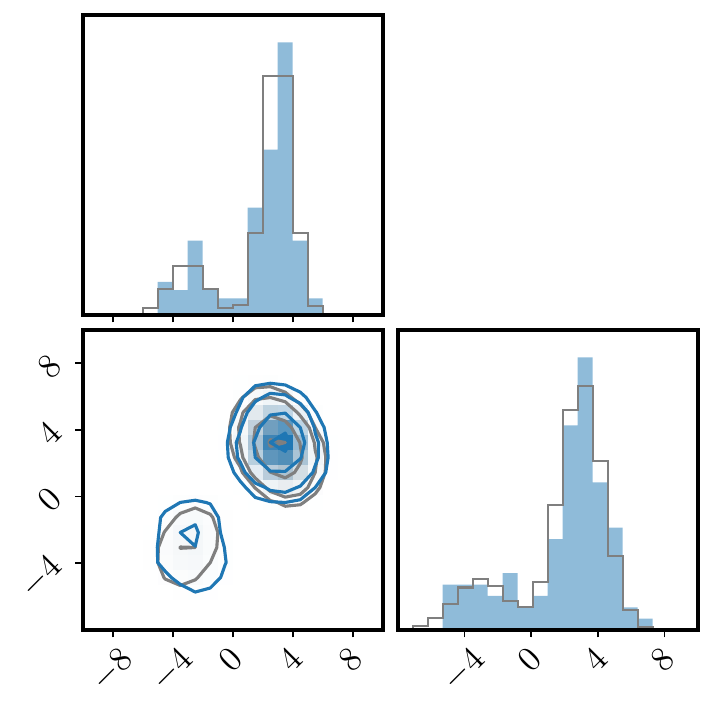}}
\end{subfigure}
\begin{subfigure}[Langevin MCMC]
{\includegraphics[width=0.49\textwidth]{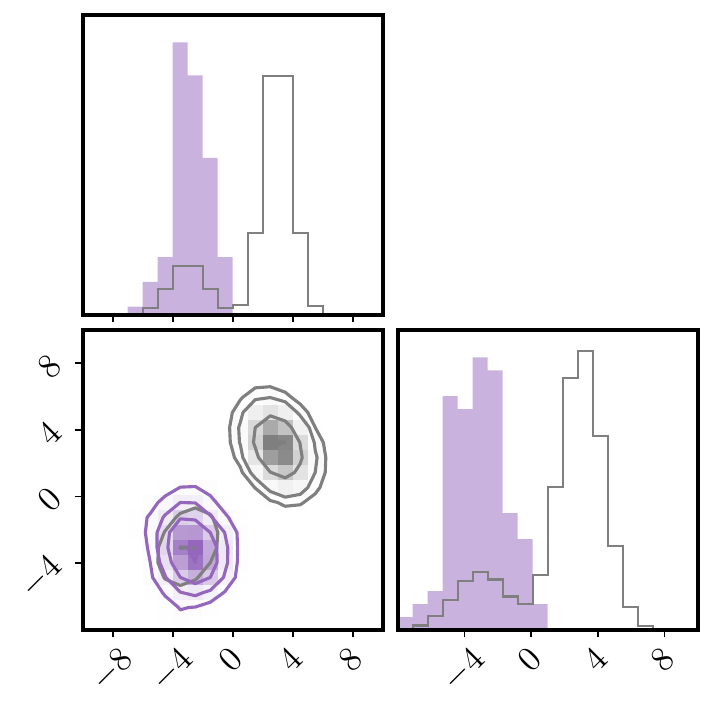}}  
\end{subfigure}
\end{center}
\caption{\label{fig:gmm_correlated} Mixture of correlated Gaussians. }
\end{figure}

\end{document}